\documentclass[11pt]{article}
\usepackage{paper}
\usepackage[top=1in,bottom=1in,left=0.8in,right=0.8in]{geometry}
\usepackage{algorithm}
\usepackage{algpseudocode}
\usepackage{wrapfig}
\usepackage{xspace}
\usepackage[backref=page]{hyperref}

\providecommand{\email}[1]{\href{mailto:#1}{\nolinkurl{#1}\xspace}}
\makeatletter
\newcommand{\TODO}[1]{\textcolor{red}{[TODO\@ifnotempty{#1}{: #1}]}}
\newcommand{\fred}[1]{\textcolor{purple}{[sandeep\@ifnotempty{#1}{: #1}]}}
\makeatother

\title{Hardness and Algorithms for Robust and Sparse Optimization}
\author{
Eric Price\thanks{Department of Electrical and
Computer Engineering, The University of Texas at Austin. \email{ecprice@gmail.com}}
\and
Sandeep Silwal\thanks{Electrical Engineering and Computer Science Department, Massachusetts Institute of Technology.
\email{silwal@mit.edu}}
\and
Samson Zhou\thanks{Computer Science Department, Carnegie Mellon University.
\email{samsonzhou@gmail.com}}}
\date{\today}

\begin{document}
\maketitle 
\allowdisplaybreaks

\begin{abstract}
We explore algorithms and limitations for sparse optimization problems such as sparse linear regression and robust linear regression. The goal of the sparse linear regression problem is to identify a small number of key features, while the goal of the robust linear regression problem is to identify a small number of erroneous measurements. Specifically, the sparse linear regression problem seeks a $k$-sparse vector $x\in\mathbb{R}^d$ to minimize $\|Ax-b\|_2$, given an input matrix $A\in\mathbb{R}^{n\times d}$ and a target vector $b\in\mathbb{R}^n$, while the robust linear regression problem seeks a set $S$ that ignores at most $k$ rows and a vector $x$ to minimize $\|(Ax-b)_S\|_2$.

We first show bicriteria, NP-hardness of approximation for robust regression building on the work of \cite{ODonnellWZ15} which implies a similar result for sparse regression. We further show fine-grained hardness of robust regression through a reduction from the minimum-weight $k$-clique conjecture. On the positive side, we give an algorithm for robust regression that achieves arbitrarily accurate additive error and uses runtime that closely matches the lower bound from the fine-grained hardness result, as well as an algorithm for sparse regression with similar runtime. Both our upper and lower bounds rely on a general reduction from robust linear regression to sparse regression that we introduce. Our algorithms, inspired by the 3SUM problem, use approximate nearest neighbor data structures and may be of independent interest for solving sparse optimization problems. For instance, we demonstrate that our techniques can also be used for the well-studied sparse PCA problem.

\end{abstract}

\section{Introduction}
Sparsity is often a key feature embedded within large datasets and fundamentals tasks in data science and machine learning. 
For example, in the sparse linear regression or variable selection problem, the goal is to find a $k$-sparse vector $x\in\mathbb{R}^d$ to minimize $\|Ax-b\|_2$, given an input matrix $A\in\mathbb{R}^{n\times d}$ and a target vector $b\in\mathbb{R}^n$. 
The intuition is that the target vector can be effectively summarized as the linear combination of a small number of columns of $A$ that denote the important features of the data. 
Similarly, in the robust linear regression or constraint selection problem, the goal is to find a set $S$ that ignores at most $k$ rows and a vector $x\in\mathbb{R}^d$ to minimize $\|(Ax-b)_S\|_2$, given an input matrix $A\in\mathbb{R}^{n\times d}$ and a target vector $b\in\mathbb{R}^n$, where the notation $(\cdot)_S$ means we only measure the loss over coordinates in $S$. 
In other words, we suppose that up to $k$ entries of the target vector can be arbitrarily corrupted and thus ignored in the computation of the resulting empirical risk minimizer. 
In both the sparse and robust linear regression problems, the $L_2$ loss function can be naturally generalized to other loss functions which align with specific goals, such as truncating or mitigating the penalty beyond a certain threshold. 

Sparsity is a desirable attribute for model design in machine learning, statistics, estimation, and signal processing. 
Simpler models with a small number of variables not only provide more ease for interpetability, but also tend to have smaller generalization error~\cite{FosterKT15}. 
A common algorithmic approach for acquiring sparse vectors for the task of sparse linear regression is to use greedy algorithms that iteratively select features, e.g., stepwise selection, backward elimination, and least angle regression. 
Another common technique for inducing sparse solutions is to penalize the objective with a regularization function. 
For example, the well-known LASSO adds a penalty term to the regression objective that is proportional to the $L_1$ norm of the underlying minimizer $x$~\cite{tibshirani1996regression} while ridge regression~\cite{hoerl1970ridge} often adds a penalty that is proportional to the $L_2$ norm of $x$. 

\subsection{Our Results}
In this paper, we study algorithms and limitations for sparse regression and robust regression. 
In particular, we consider the following problems: 

\begin{problem}[Robust Regression]
\label{prob:robust:reg}
Given $A \in \R^{n \times d}, b \in \R^n$, a loss function $\calL:\R^*\to\R$, and integer $0 < k \le n$, find $T \subset [n]$ satisfying $|T|\le k$ and $x\in\R^n$ to minimize $\calL((Ax-b)_T)$, where $(Ax-b)_T$ denotes that we only measure the loss on the coordinates in $T$. The coordinates not in $T$ are called \emph{ignored}.
\end{problem}

\begin{problem}[Sparse Regression]
\label{prob:sparse:reg}
Given $A \in \R^{n \times d}, b \in \R^n$, a loss function $\calL:\R^*\to\R$, and integer $0 < k \le n$, find $x\in\R^n$ with $\|x\|_0\le k$ to minimize $\calL(Ax-b)$.
\end{problem}
\cite{ODonnellWZ15} showed the hardness of approximation for maximizing the number of satisfied linear equations; their result can be translated to show bicriteria robust regression is NP-hard, i.e., it is NP-hard to achieve a multiplicative factor approximation to the optimal loss while allowing a constant factor larger parameter for sparsity. 
We introduce a further reduction to show the NP-hardness of bicriteria sparse regression. 

\begin{theorem}
\label{thm:sparse:bicrit:hard}
Let $\calL:\mathbb{R}^n\to\mathbb{R}$ be any loss function such that $\calL(0^n)=0$ and $\calL(x)>0$ for $x\neq 0^n$. 
Given a matrix $X \in \R^{a \times b}$, vector $c \in \R^{a}$, and a sparsity parameter $t$, let $OPT = \min_w\calL(Xw-c)$ where $\|w\|_0 = t$.
Then for any $C_1 > 1$ and any constant $C_2 > 1$, it is NP-hard to find $w'$ satisfying $\|w'\|_0 = C_2t$ such that $\calL(Xw'-c) \le C_1 \cdot OPT$.
\end{theorem}

We remark that the statement of Theorem~\ref{thm:sparse:bicrit:hard} holds for any parameter $C_1>1$, showing that multiplicative approximation is NP-hard, even for arbitrarily large multiplicative factors that could depend on $n$ and $d$. 
It also applies to a large number of commonly used loss functions, such as $L_p$ loss or various $M$-estimators, e.g.,~\cite{ClarksonW15,ClarksonWW19,TukanWZBF22}. 
See Table~\ref{tab:loss} for more details.  
For completeness, we also formalize the proof of \cite{ODonnellWZ15} to show bicriteria hardness of robust regression. 

\begin{table}[!htb]
\centering
\begin{tabular}{c|c}\hline
Loss function & Formulation \\\hline
$L_p$ & $|x|^p$ \\\hline
Cauchy & $\frac{\lambda^2}{2}\log\left(1+(x/\lambda)^2\right)$ \\\hline
Fair & $\lambda|x|-\lambda^2\ln\left(1+\frac{|x|}{\lambda}\right)$ \\\hline
Geman-McClure & $\frac{x^2}{2 + 2x^2}$  \\\hline
Huber &  $\begin{cases} x^2/2 & \text{if }|x| \leq \lambda\\
\lambda|x| - \lambda^2/2 & \text{otherwise} \end{cases}$ \\\hline
$L_1-L_2$ & $2\left(\sqrt{1+\frac{x^2}{2}}-1\right)$ \\\hline
Tukey & $\begin{cases} \frac{\lambda^2}{6}\left(1-\left(1-\frac{x^2}{\lambda^2}\right)^3\right) & \text{if } |x| \leq \lambda\\
\frac{\lambda^2}{6} & \text{otherwise} \end{cases}$ \\\hline
Welsch & $\frac{\lambda^2}{2}\left(1-e^{-\left(\frac{x}{\lambda}\right)^2}\right)$ \\\hline
\end{tabular}
\caption{Theorem~\ref{thm:sparse:bicrit:hard} shows the bicriteria hardness of approximation for sparse regression for common loss functions. The same bicriteria hardness also holds for robust regression.}
\label{tab:loss}
\end{table}

To prove Theorem~\ref{thm:sparse:bicrit:hard}, we prove a reduction (see Corollary \ref{cor:general_reduction}) which states that an algorithm for the sparse regression problem can be used to solve the robust regression problem using a polynomial time blow-up. 
However, it is known that the sparse regression problem requires runtime $O(n^k)$ under the minimum-weight $k$-clique conjecture~\cite{gupte2021finegrained}; could it be the case that robust regression is significantly easier? 
We give a fine-grained hardness result showing this is not the case:

\begin{theorem}
\label{thm:finegrain:robust}
For every $\eps>0$, there exists a sufficiently large $n$ such that the robust regression problem requires $ \Omega(n^{k/2+o(1)})$ randomized time, unless the minimum-weight $k$-clique conjecture is false.
\end{theorem}

To complement our fine-grained hardness results, we develop nearly-matching upper bounds with \emph{additive error} in the case there exists a diagonal $S$ ignoring $k$ rows achieving zero loss. Namely, we give an algorithm with time $O(nk(n/\eps)^{Ck})$, where $C<1$ can be any constant arbitrarily close to $\frac{1}{2}$, thereby nearly matching the lower bounds of Theorem~\ref{thm:finegrain:robust}. 

\begin{theorem}\label{thm:robust_ub:informal}
Given $A \in \R^{n \times d}$ and $b \in \R^n$ such that \[\min_{S,x} \|S(Ax-b)\|_2 = 0 \]
over all diagonal matrices $S$ with $n-k$ one entries in the diagonal and $k$ zero entries, there exists an algorithm that  returns a diagonal matrix $S'$ with $n-k$ one entries in the diagonal and $k$ zero entries and a vector $x'$ such that 
$\|S'(Ax'-b)\|_2 \le \eps$ in time 
\[ \min_{c \ge 1} O\left( nk \cdot \left( \frac{12c \cdot n  \cdot \|b\|_2}{\eps}  \right)^{\frac{k}{2} \cdot (1+ 1/(2c^2-1))} \right). \]
\end{theorem}

We obtain an algorithm with a similar guarantee for sparse regression.

\begin{theorem}
\label{thm:sparse:planted:informal}
Given $A \in \R^{n \times d}$ and $b \in \R^n$ such that there exists a $k$-sparse vector $x$ satisfying $Ax = b$, there exists an algorithm that returns a $k$-sparse vector $z\in\R^d$ satisfying $\|Az-b\|_2 \le \eps$ in time 
\[ \min_{c \ge 1} O\left( nk \cdot \left( \frac{12c \cdot d \cdot \|b\|_2}{\eps}  \right)^{\frac{k}{2} \cdot (1+ 1/(2c^2-1))} \right). \]
\end{theorem}

Our algorithms for Theorems~\ref{thm:sparse:planted:informal} and \ref{thm:robust_ub:informal} are inspired by techniques for the 3SUM problem and its generalization to $k$ integers, in which the goal is to determine whether a set of $n$ integers contain $k$ integers that sum to zero. 
Rather than check all $O(n^k)$ possible $k$-sparse vectors, we first input all $\frac{k}{2}$-sparse vectors into an approximate nearest neighbor (ANN) data structure. 
Given a query $q$, the ANN data structure will output a point $x$ such that $\|x-q\|_2\le c\cdot\min_y\|y-q\|_2$, where the minimum is taken over all points that are input into the data structure.  
We can thus query the ANN data structure over all $\binom{n}{k/2}$ differences between the measurement vector $b$ and all  $\frac{k}{2}$-sparse vectors, reconstructing the $k$-sparse vector from the query that achieves the minimum value. 

Surprisingly, our technique also works for the sparse principal component analysis (PCA) problem, in which the goal is to find a $k$-sparse unit vector $v\in\R^n$ to maximize $v^T Av$, given an input PSD matrix $A\in\R^{n\times n}$ of rank $r$. 

\begin{theorem}\label{thm:sparse_pca_ub_exact:informal}
There exists an algorithm that uses $\widetilde{O}\left(\frac{k^2}\eps \cdot \left( r\left( \frac{n\kappa}{\eps} \right)^{k(1+\eps)/2} + \left( \frac{1}{\eps} \right)^{r/2 + 1} \right) \right)$ time and with high probability, outputs a $k$-sparse unit vector $u$ such that with high probability,
\[u^T Au \ge (1-\eps) \max_{\|v\|_2 = 1 \,,  \|v\|_0 \le k} v^T Av,\]
where $\kappa$ is the condition number of $A$.
\end{theorem}

We also give a simple NP-hardness proof of the robust regression problem in the appendix using a reduction from the exact cover problem, in which the goal is to determine whether there exists a sub collection $S'$ (of an input collection $S$ of subsets of $X$) such that every member of $X$ belongs to exactly one set in $S'$.
Finally, we give explicit examples of why natural algorithms such as the greedy algorithm or alternating minimization fail to achieve multiplicative guarantees.

\subsection{Prior Works}
\paragraph{Robust regression.} 
Robust regression has been well-studied in recent years. 
However, virtually all works rely on \emph{distributional} assumptions on the input. 
The goal is to statistically recover the coefficient vector which minimizes the expected loss given a small number of corruptions to the distributional input. 
In contrast, we assume no distributional assumptions at all and our goal is to solve the computational problem given a corrupted input. 
Distributional assumptions have been well-studied in part because they are tractable; see \cite{KlivansKM18,KarmalkarKK19,DiakonikolasKS19,Cherapanamjeri2020, zhu2020, BakshiP21, jambulapati2021robust}
and the references within for a more comprehensive overview of distributional works.

Indeed, many such works such as \cite{BhatiaJK15, Bhatia0KK17, SuggalaBR019} state their main motivation behind using distributional assumptions is that they believe the computational version, which we study, to be `hard.' 
Since the focus of these works is statistical in nature, they do not rigorously justify why the robust regression problem in general is computationally hard. 
Our work aims to fill in this gap and initiate the study of hardness for robust regression. 
We remark that some works such as \cite{BhatiaJK15, Bhatia0KK17, SuggalaBR019} assert that a proof of NP-hardness of robust regression is given in \cite{StuderKPB12}. 
However, it seems \cite{StuderKPB12} actually does not study robust regression at all; instead, it seems they study the problem of \emph{sparse} regression, which, although related, does not imply anything about the hardness of robust regression. 
In fact, one of our results (Corollary~\ref{cor:general_reduction}) gives a reduction from robust regression to sparse regression, which implies that robust regression is a strictly easier problem than sparse regression. 
On the other hand, it is not clear whether hardness for sparse regression implies anything about hardness for robust regression. 

\paragraph{Sparse regression.}
The sparse regression problem has also recently received significant attention, e.g.,~\cite{Natarajan95,davis1997adaptive,Mahabadi15,FosterKT15,Har-PeledIM18,ChenYW19,gupte2021finegrained}. 
\cite{Natarajan95} showed the NP-harndess of sparse regression while \cite{FosterKT15} showed that assuming SAT cannot be solved by a deterministic algorithm in $O(n^{\log\log n})$ time, then no polynomial-time algorithm can find a $k'$-sparse vector $x$ with $\|Ax-b\|_2\le\poly(n)$, for $k'=k\cdot 2^{\log^{1-\delta} n}$. 
Subsequently, \cite{Har-PeledIM18} showed the sparse regression required $\Omega(n^{k/2})$ time, assuming the $k$-SUM conjecture from fine-grained complexity. 
\cite{gupte2021finegrained} strengthened this lower bound to $\Omega(n^{k-\eps})$ time, for any constant $\eps>0$, using the minimum-weight $k$-clique conjecture. 

$L_1$-relaxation based algorithms such as basis pursuit~\cite{ChenDS98}, Lasso~\cite{tibshirani1996regression}, and the Dantzig selector~\cite{candes2007dantzig} have been developed for practical usage that do not involve worst-case inputs. 
For example, they consider the setting $b=Ax+g$, where the noise vector $g$ is drawn from a Gaussian distribution and the design matrix $A$ is well-conditioned. 
There has also been an extensive study on other penalty classes, such as the smoothly clipped absolute deviation penalty~\cite{fan2001variable}, the $L_p$ norm for bridge estimators~\cite{frank1993statistical}, or as the regularization function for $M$-estimators~\cite{loh2015regularized}. 
\cite{ChenYW19} showed the NP-hardness of $O(n^{C_1}d^{C_2})$ multiplicative approximation for these common regularizations of the sparse regression problem and fixed constants $C_1,C_2>0$, when the loss function is convex and the penalty function is sparse, such as $L_1$-relaxation. 
By comparison, we show bicriteria NP-hardness of \emph{any} multiplicative approximation of the actual sparse regression problem, even when the sparsity constraint can be relaxed up to a multiplicative factor. 

\paragraph{Other sparse optimization problems.}
Sparsity has also been highly demanded in other optimization problems. 
In this paper, we show that our algorithmic ideas also extend to the sparse PCA problem, which was first introduced by \cite{dAspremontGJL07} and subsequently shown to be NP-hard by \cite{MoghaddamWA06}. 
In fact, it is NP-hard to obtain any multiplicative approximation if the input matrix is not PSD~\cite{Magdon-Ismail17} and to obtain a $(1+\eps)$-multiplicative approximation when the input matrix is PSD~\cite{ChanPR16}, though \cite{AsterisPKD15} gave an \emph{additive} polynomial time approximation scheme based on the bipartite maximum weight matching problem. 
In practice, techniques for the more general PCA problem based on rotating previously studied PCA approaches
based on rotation~\cite{jolliffe1995rotation} or thresholding~\cite{cadima1995loading} the top singular vector of the input matrix seemed to suffice for specific applications. 
$L_1$ relaxations~\cite{jolliffe2003modified} and similar heuristics~\cite{zou2005regularization,zou2006sparse,shen2008sparse} have also been considered for the sparse PCA problem. 
Another line of direction considered semidefinite programming relaxations~\cite{d2008optimal,amini2008high,dOrsiKNS20,ChowdhuryDWZ20,ChowdhuryBZWD22}. 

Our work also connects to the problem of recovering the sparsest non-zero element in a linear subspace problem (see Theorem \ref{thm:robust:planted}). This problem is known to be NP-hard in the worst case \cite{coleman1986null}. On the positive side, there exists work on planted settings of the problem where the subspace is generated by the span of random Gaussian vectors along with a planted sparse vector; see \cite{DemanetH14} and references within.

\section{Bicriteria Hardness of Approximation}
In this section, we show the bicriteria hardness of approximation for both robust regression and sparse regression. 
Our results also generalize to loss functions that have no penalty on the zero vector and positive penalty on any nonzero vector. The bicriteria hardness result for robust regression is immediately implied by the results in \cite{ODonnellWZ15}, but it is not phrased in terms of the robust regression problem. We formalize the details below for completeness. In addition, we extend the bicriteria hardness result for sparse regression in the following section.

\subsection{Bicriteria Hardness of Approximation for Robust Regression}
\begin{definition}[$\maxklin$]
Suppose there exist a list of $n$ linear equations of the form  $a_1x_{i_1}+\ldots+a_kx_{i_k}=b$, where $a_1,\ldots,a_k,b$ are constants from a ring $R$ and $x_{i_1},\ldots,x_{i_k}$ are variables from the set $x_1,\ldots,x_d$. 
Then the goal of the $\maxklin(R)$ problem is to assign values in $R$ to the variables $x_1,\ldots,x_d$ such to maximize the total number of satisfied linear equations. 
\end{definition}

\begin{definition}[$\bgmaxthreelin$]
Suppose there exist a list of $n$ linear equations of the form  $x_{i_1}+x_{i_2}-x_{i_3}=b$, such that $|b|\le B$ for some fixed $B\in R$ and $x_{i_1},x_{i_2},x_{i_3}$ are variables from the set $x_1,\ldots,x_d$. 
Then the goal of the $\bgmaxthreelin(R)$ problem is to assign values in $R$ to the variables $x_1,\ldots,x_d$ such to maximize the total number of satisfied linear equations. 
\end{definition}
We use the notation $OPT_R(I)$ to denote the maximum fraction of
equations of an instance $I$ that can be satisfied when the equations are evaluated over $R$.

\begin{theorem}[Theorem 1.2 in \cite{ODonnellWZ15}, Hardness of Approximation of $\bgmaxthreelin(R)$]
For all constants $\eps,\kappa\in(0,1)$ and $q\in\mathbb{N}$, given an instance of $\bgmaxthreelin(R)$, it is NP-hard to distinguish whether
\begin{itemize}
\item Completeness: There is a $(1-\eps)$-good assignment over $\mathbb{Z}$, i.e., $OPT_{\mathbb{Z}}(I)\ge(1-\eps)$.
\item Soundness: There is no $(1/q + \kappa)$-good assignment over $\mathbb{Z}_q$, i.e., $OPT_{\mathbb{Z}_q}(I)\le\frac{1}{q}+\kappa$. \end{itemize}
\end{theorem}
Although there seems to be a typo in the statement of Lemma A.1 in~\cite{ODonnellWZ15}, the corresponding proof provides the following guarantee:
\begin{lemma}[\cite{ODonnellWZ15}]
Given an instance $I$ of $\bgmaxthreelin$, $OPT_{\mathbb{R}}(I)\ge\frac{1}{8}\,OPT_{\mathbb{Z}_q}(I)$. 
\end{lemma}
By setting $\kappa$ in the following statement to be $\frac{1}{8q}+\frac{1}{8}\kappa$ in the above formulations, we have that
\begin{corollary}\label{cor:odonell}
For all constants $\eps\in(0,1),\kappa\in(0,1/8)$, given an instance of $\bgmaxthreelin$, it is NP-hard to distinguish whether
\begin{itemize}
\item Completeness: There is a $(1-\eps)$-good assignment over $\mathbb{Z}$, i.e., $OPT_{\mathbb{Z}}(I)\ge(1-\eps)$.
\item Soundness: There is no $\kappa$-good assignment over $\mathbb{R}$, i.e., $OPT_{\mathbb{R}}(I)\le\kappa$. 
\end{itemize}
\end{corollary}

\begin{theorem}\label{thm:sparse_constraints_bicriteria}
Given a matrix $A\in\mathbb{Z}^{n\times d}$, a sparsity parameter $k$, and a vector $b\in\mathbb{R}^n$, let $OPT=\min_{S,x}\|SAx-Sb\|_2$, where the minimum is taken over all diagonal matrices $S$ that have $k$ entries that are zero and $n-k$ entries that are one and all $x \in \R^d$. 
Then for any $C_1>1$ (which can depend on the parameters $n,d$) and any constant $C_2>1$, it is NP-hard to find a matrix $S'$ with $C_2k$ entries that are zero and $n-C_2k$ entries that are one and a vector $x\in\mathbb{R}^d$ such that $\|S'Ax-S'b\|_2\le C_1\cdot OPT$.
\end{theorem}
\begin{proof}
Let $\kappa$ and $\eps$ be constants so that $\frac{1-\kappa}{\eps}\ge C_2$. 
Given an instance $I$ of $\bgmaxthreelin$, set $k=\eps n$, where $n$ is the number of linear equations over the variables $x_1,\ldots,x_d$. 
We create the corresponding $n\times d$ matrix $A$ by setting $A_{i,j}\pm 1$ if the coefficient of $x_j=\pm 1$ in the $i$-th linear equation. 
Otherwise, we set $A_{i,j}=0$. 

Observe that if there is a $(1-\eps)$-good assignment over $\mathbb{Z}$, then there is a vector $x\in\mathbb{R}^d$ that satisfies $(1-\eps)n$ linear equations. 
Thus by ignoring $k=\eps n$ linear equations, the remaining coordinates of $b$ are satisfied by the vector $x$. 
Hence, there exists a matrix $S$ with $k$ entries that are zero and $n-k$ entries that are one such that $SAx=Sb$ and in particular, $\|SAx-Sb\|_2=0$ and $OPT=0$. 

On the other hand, if there is no $\kappa$-good assignment over $\mathbb{R}$ to $I$, then any vector $x\in\mathbb{R}^d$ will satisfy fewer than $\kappa n$ linear equations. 
In particular, even by ignoring $C_2k\le(1-\kappa)n$ linear equations, there still exists some coordinate of $b$ that is not satisfied by the vector $x$. 
Thus for every matrix $S'$ with $C_2k$ entries that are zero and $n-C_2k$ entries that are one, we have $S'Ax\neq Sb'$ and in particular, $\|S'Ax-S'b\|_2>0$ and therefore $\|S'Ax-S'b\|_2>C_1\cdot 0$ for any $C_1>1$. 

Hence, any algorithm that finds a matrix $S'$ with $C_2k$ entries that are zero and $n-C_2k$ entries that are one, as well as a vector $x\in\mathbb{R}^d$ such that $\|S'Ax-S'b\|_2\le C_1\cdot OPT$ can differentiate whether (1) there is a $(1-\eps)$-good assignment over $\mathbb{Z}$, i.e., $OPT_{\mathbb{Z}}(I)\ge(1-\eps)$ or (2) there is no $\kappa$-good assignment over $\mathbb{R}$, i.e., $OPT_{\mathbb{R}}(I)\le\kappa$. 
Therefore by Corollary~\ref{cor:odonell}, it is NP-hard to find a matrix $S'$ with $C_2k$ entries that are zero and $n-C_2k$ entries that are one and a vector $x\in\mathbb{R}^d$ such that $\|S'Ax-S'b\|_2\le C_1\cdot OPT$.
\end{proof}
Finally, we observe that in the proof of Theorem~\ref{thm:sparse_constraints_bicriteria}, in one case, the resulting loss is the all zeros vector while in the other case, there are at least $n-C_2k$ nonzero entries. 
Thus the proof generalizes to any loss function $\calL$ such that $\calL(0^n)=0$ and $\calL(x)>0$ for $x\neq 0^n$. 

\subsection{Bicriteria Hardness of Approximation for Sparse Regression}
We now extend the bicriteria hardness results from the previous section for the problem of sparse regression.  Note that in the sparse regression problem, including as many variables as possible only helps us. For example, including as many columns as possible in linear regression only enlarges the column space and hence finds a possibly closer vector to the target. In the linear regression case, we prove the following theorem which states that it is NP-hard to choose a set of columns to ignore, even if we relax the number of ignored columns by a multiplicative factor.

\begin{theorem}\label{thm:approx_hardness_cs}
Given a matrix $X \in \R^{a \times b}$, vector $c \in \R^{a}$, and a sparsity parameter $t$, let $OPT = \min_w \|Xw-c\|$ where $\|w\|_0 = t$. Then for any $C_1 > 1$ and any constant $C_2 > 1$, it is NP-hard to find $w'$ satisfying $\|w'\|_0 = C_2t$ such that $\|Xw'-c\| \le C_1 \cdot OPT$.
\end{theorem}
\begin{proof}
We will show that solving the sparse regression problem allows us to solve the robust regression problem. 
Let $A \in \mathbb{Z}^{n \times d}, b \in \R^n$, and $k$ be the parameters in Theorem \ref{thm:sparse_constraints_bicriteria}. 
Note that the matrix $A$ in the proof of Theorem \ref{thm:sparse_constraints_bicriteria} satisfies $n > d$. 
Given $A$, we construct the matrix $X\in\R^{(n-d)\times n}$ so that the rows of $X$ will span the space that is orthogonal to the column subspace of $A$, i.e., $XA = 0$. 
We also let $c = -Xb$, and $t = k$. 
We claim that if we can solve the sparse regression problem with this derived instance, then we can solve the robust regression problem of Theorem \ref{thm:sparse_constraints_bicriteria}. 
Parameters $C_1$ and $C_2$ are the same in both theorems.

First, note that the $OPT$ value in Theorem \ref{thm:sparse_constraints_bicriteria} is equal to $0$: we can ignore $k$ linear constraints and get $0$ loss. This means that there exists a diagonal matrix $S$ and a vector $x$ such that $SAx - Sb = 0$ which implies $Ax-b = w$ for $\|w\|_0 = k$ i.e., $w$ has $n-k$ zero entries. Therefore, the value of $OPT$ in the sparse regression problem is also $0$ by multiplying the equation $Ax-b = w$ by the matrix $X$ since $X(Ax-b) = Xw$ implies $Xw = c$. Now suppose we can solve the sparse regression problem with any factor $C_1' > 1$ approximation while satisfying $\|w'\|_0\le C_2'k$. Since we are assuming a multiplicative approximation factor, $w'$ must also evaluate to $0$ loss in our objective. Furthermore, $w'$ has  $n-C_2'k$ zero entries. 

Now $Xw' = c$ which is the same as $X(w'+b) = 0$. Thus, $w'+b$ lies in the orthogonal complement of the rows of $X$ by our construction of $X$. Hence, $w'+b$ lies in the column space of $A$. Therefore, there exists some $x'$ such that $Ax' = w'+b$ or in other words, $Ax'-b = w'$. Letting $S'$ have the diagonal which is the indicator for the sparsity of $w'$, we get that $S'$ has $C_2k$ zero entries on the diagonal. This implies $\|S'Ax'-S'b\|_2 = 0 \le C_1 \cdot OPT$. Since finding such a pair $S',x'$ is NP-hard from Theorem  \ref{thm:sparse_constraints_bicriteria}, it must be the case that the sparse regression selection problem stated in the current theorem statement must also be NP-hard, as desired.
\end{proof}

Observe that in the proof of Theorem~\ref{thm:approx_hardness_cs}, in one case, we have $Xw=c$ so that the resulting loss is the all zeros vector while in the other case, the resulting vector is nonzero. 
Therefore, the proof of Theorem~\ref{thm:approx_hardness_cs} generalizes to any loss function $\calL$ such that $\calL(0^n)=0$ and $\calL(x)>0$ for $x\neq 0^n$, giving Theorem~\ref{thm:sparse:bicrit:hard}. 

Note that the reduction given in the above proof implies the following general statement: robust regression is `easier' than sparse regression selection. That is, if there exists an algorithm for sparse regression, we can use it to solve robust regression as well.

\begin{corollary}\label{cor:general_reduction}
Let $\mathcal{A}$ be an algorithm which solves the following sparse regression problem:
\[ \min_w \|Xw-c\|_2 \, \, \, \text{s.t.} \, \, \, \|w\|_0 = t \]
in time $f(X,c,t)$. Consider the robust regression problem of 
\[\min_{S,x} \|S(Ax-b)\|_2 \]
where $S$ is constrained to be a diagonal matrix with $n-k$ one entries in the diagonal and $k$ zero entries. $\mathcal{A}$ solves this problem in time $f(X', c', k)$ where $X'$ is any matrix satisfying $X'A = 0$ and $c' = -X'b$.
\end{corollary}
\begin{proof}
The proof follows from the proof of Theorem \ref{thm:approx_hardness_cs}.
\end{proof}

\section{Fine-Grained Hardness}

In this section, we prove a fine-grained hardness result for robust regression. We first need the following definition and theorem from \cite{gupte2021finegrained}.

\begin{definition}[Definition $1$ in \cite{gupte2021finegrained}, $k$-SLR$_p$]
For any integer $k \ge 2$ and $1 \le p \le \infty$,  the $k$-sparse regression problem with respect to the $\ell_p$ norm is defined as follows. Given a matrix $A \in \R^{M \times N}$, a target vector $b \in \R^M$, and a number $\delta > 0$, distinguish between:
\begin{itemize}
    \item a \textbf{YES} instance: there is some $k$-sparse $x \in \R^N$ such that $\|Ax-b\|_p \le \delta$, and
    \item  a \textbf{NO} instance: for all $k$-sparse $x \in \R^N$, $\|Ax-b\|_p > \delta$.
\end{itemize}
\end{definition}

\begin{theorem}[Theorem $3$ in \cite{gupte2021finegrained}]\label{thm:gupte_min_weight}
For any integer $k \ge 4$, the $k$-SLR$_2$ problem requires time $\Omega(N^{k-o(1)})$ (randomized) time, unless the min-weight-$k$-clique conjecture is false.
\end{theorem}

Using this theorem (or rather its proof), we can prove the hardness of the following decision version of sparse regression: 
\begin{problem}
\label{prob:robust:decision}
Given a matrix $A \in \R^{M \times N}$, a target vector $b \in \R^M$, integer $0 < k \le M$, and a number $\delta > 0$, distinguish between:
\begin{itemize}
    \item a \textbf{YES} instance: there is some diagonal matrix $S$ with $k$ zeros on the diagonal and $n-k$ ones on the diagonal and some $x \in \R^N$ such that $\|S(Ax-b)\|_2 \le \delta$, and
    \item  a \textbf{NO} instance: for all diagonal matrices $S$ with $k$ zeros and $n-k$ ones on the diagonal and all $x \in \R^N$, $\|S(Ax-b)\|_2 > \delta$.
\end{itemize}
\end{problem}

We first recall the following theorem of \cite{gupte2021finegrained} and the key underlying details of the proof, which we encapsulate in the following theorem.

\begin{theorem}
\cite{gupte2021finegrained}
\label{thm:gupte_explained}
Let $G = (V,E)$ be a graph with $N$ vertices with $w_e$ denoting the integer weight of edge $e$. Let $W, k$ be integer parameters and $Z = |\{(u, v) \not \in E | u, v \in V \}|$ to be the number of non-edges in the graph $G$. Suppose that we know a partition of the $N$ vertices into $k$ blocks of size $N/k$ such that if a $k$-clique of weight at most $W$ exists then there is such a $k$-clique with exactly one vertex in each block \footnote{This is a standard assumption and is without loss of generality; see Section $2.2$ of \cite{gupte2021finegrained}. We can assume the $k$ blocks are $N/k$ consecutive integers of $\{1,\ldots, N\}$ in order.}. Set 
\[\alpha = \sqrt{\max\left(1, \sum_{e \in E} w_e + 8W\right)} , \beta = \sqrt{\sum_{e \in E} w_e +8W + \alpha^2 Z } \cdot \max\left(8Z, 50\left(\alpha^2Z, + \sum_{e \in E} w_e \right) \right).  \]
Define the matrix $A = \begin{pmatrix}
  C\\ 
  D
\end{pmatrix}$ and the vector $b = \begin{pmatrix}
  c\\ 
  d
\end{pmatrix}$ as follows:
$C \in \R^{\binom{N}2 \times N}$ with rows indexed by all unordered pairs of possible edges and columns indexed by vertices. For a possible edge $e = (u,v)$ not in $G$, the row of $C$ corresponding to $e$ has $2\alpha$ in the columns of $u$ and $v$ and the corresponding entry of $c$ has $\alpha$. If $e = (u,v) \in E$ then the columns of $u$ and $v$ have the entry $2\sqrt{w_e}$ and the corresponding entry of $c$ has $\sqrt{w_e}$. All other entries of $C$ are $0$. We also have $D \in \R^{k \in N}$
\[ D = \begin{pmatrix}
  \beta 1 & 0  & \cdots & 0\\ 
  0 & \beta 1  & \cdots & 0 \\
  \cdots & \cdots & \cdots & \cdots \\
  0 & 0 & \cdots & \beta 1
\end{pmatrix}, \qquad d = \begin{pmatrix}
  \beta\\ 
  \beta \\
  \vdots \\
  \beta
\end{pmatrix} \]
where $1$ denotes a row vector of length $N/k$ of all ones. Finally, set $\delta =\sqrt{\sum_{e \in E} w_e + 8W + \alpha^2Z } > 0$. The following statements hold about $A$ and $b$:
\begin{itemize}
    \item (Completeness) If $G$ contains a $k$-clique of weight at most $W$ then we can let $x$ be the indicator vector of the clique and we have $\|Ax-b\|_2 \le \delta$.
    \item (Soundness) If there exists $x$ such that $\|Ax-b\|_2 \le \delta$ then $x$ must be the indicator vector of a $k$-clique in $G$ of weight at most $W$ and furthermore, each block of $N/k$ vertices must have exactly one vertex in the clique. 
\end{itemize}
\end{theorem}

We now formalize the conditions and provide the proof of Theorem~\ref{thm:finegrain:robust}. 
\begin{theorem}
For every $\eps > 0$, there exists a sufficiently large $M$ such that 
Problem~\ref{prob:robust:decision} requires $ \Omega(M^{k/2+o(1)})$ randomized time, unless the min-weight-$k$-clique conjecture is false.
\end{theorem}
\begin{proof}
Consider the hard instance of Theorem \ref{thm:gupte_min_weight} and \ref{thm:gupte_explained}, given by $A \in \R^{M \times N}, b \in \R^M, \delta > 0$, and parameter $k$ where we set $M = \Theta(\binom{N}2)$, specified fully later. Note that $A$ in the proof in \cite{gupte2021finegrained} is constructed from a min-weight-$k$-clique graph instance $G$ and is composed of matrices $C,D$ in their notation (see the statement of Theorem \ref{thm:gupte_explained}). We will augment $A$ in two ways. First, note that $D$ there is a $k \times N$ matrix. We will copy each row of $D$ so that each row is copied $k+1$ times. The corresponding part of the $b$ vector is also copied. We will also add $N$ additional rows to $A$: one additional row $\tilde{C} \cdot x_i=0$ for each $i \in [N]$ for some polynomially large factor $\tilde{C}$. The parameter $k$ in Problem~\ref{prob:robust:decision} will be equal to the same $k$ in Theorem \ref{thm:gupte_min_weight}.

We now mimick the proof of Theorem \ref{thm:gupte_min_weight} in \cite{gupte2021finegrained}. We first handle the easier direction. We claim that if there is a $k$-clique of weight at most $W$ for the $k$-clique instance corresponding to $A$, then we are in the \textbf{YES} case of Problem~\ref{prob:robust:decision}. Indeed, the proof of Theorem \ref{thm:gupte_min_weight} in \cite{gupte2021finegrained}, summarized in Theorem \ref{thm:gupte_explained}, shows that we can let $x$ be the indicator vector for the $k$-clique and the loss is at most $\delta$ by ignoring $k$ of the $\tilde{C} \cdot x_i = 0$ constraints, setting the value of the ignored variables to be equal to $1$, and letting the rest of the variables (which are not ignored) be equal to $0$. The matrix $D$ will contribute $0$ loss, even with the augmentation.

We now handle the remaining case. We claim that if there exists $S,x$ as in the \textbf{YES} case of Problem~\ref{prob:robust:decision}, then a $k$-clique of weight at most $W$ exists in $G$. (Soundness case of Theorem \ref{thm:gupte_explained}). Let's focus on the $D$ matrix and point out useful facts. Each row of $D$ was copied sufficiently many times such that we can't ignore all the copies of any row of $D$. Furthermore, the $x$ vector is partitioned into $k$ blocks of size $N/k$ and each row of $D$ represents the constraint that one of these $N/k$ blocks sums to $1$. Lastly, the sum of each block must be in the range $[1-2\delta/\beta, 1+2\delta/\beta]$ (for the parameter $\beta$ defined in Theorem \ref{thm:gupte_explained}) since otherwise, some row of $D$ will already give loss at least $\delta$, which cannot happen since we are assuming we are in the  \textbf{YES} case of Problem~\ref{prob:robust:decision}. This implies the following two statements. 

First, each of the $N/k$ (consecutive) blocks of $x$ must have at least $1$ ignored variable. That is, for each block of $N/k$ coordinates, there must be some $i$ belonging to the block that has its corresponding $\tilde{C} \cdot x_i = 0$ constraint ignored. Otherwise, we will incur a loss larger than $\delta$: the sum of the variables in each block must be $\Omega(1)$ due to the prior paragraph (otherwise we get a large loss for a row of $D$) but then summing the constraints $\tilde{C} \cdot x_i = 0$ for a block gives us a large loss. The only way to avoid this is to ignore one of the $\tilde{C} \cdot x_i = 0$ for an index $i$ in a block. Since we can only ignore $k$ variables and there are $k$ blocks, it follows that each block has one ignored variable and overall, we only ignore the constraints of the form $\tilde{C} \cdot x_i = 0$. The un-ignored variables can now be made arbitrarily (polynomially) small in absolute value by setting $\tilde{C}$ to be sufficiently large since otherwise the constraint  $\tilde{C} \cdot x_i = 0$ contributes more than $\delta$ loss.

The second statement we claim is that the ignored variable in each block must have its value in the range $[1-2 \delta/\beta, 1+2 \delta/\beta]$. This follows from the exact same reasoning used in the proof of Theorem \ref{thm:gupte_min_weight}. To summarize, if the ignored variable is outside this range, then the corresponding row of $D$ will induce loss larger than $\delta$ since we require the sum of all variables in a block to be close to $1$. In our case, we can potentially get some contribution from the variables not ignored. However, as stated previously, their contribution to the loss can be made to be polynomially small by setting $\tilde{C}$ large enough.

Given these two statements, the proof of the reduction follows exactly as in \cite{gupte2021finegrained}: the $k$ ignored variables, one in each block, will correspond to the $k$-clique instance. To recap the argument, the same steps as in the proof of Theorem \ref{thm:gupte_min_weight} (and \ref{thm:gupte_explained} in \cite{gupte2021finegrained}) show that we must have $(u,v) \in E$ for every $x_u, x_v$ that are ignored. Otherwise, the loss of $\|Ax-b\|_2$ coming from the non-edge terms in the matrix $C$ is much larger than $\delta$ already (see Page $8$ in \cite{gupte2021finegrained}). In addition, the proof of the fact that the $k$-clique instance has weight at most $W$ also follows. This is because our vector $x$ satisfies the same conditions required for the soundness case of Theorem \ref{thm:gupte_explained}; this portion of the proof in \cite{gupte2021finegrained} proceeds by lower bounding the error from the matrix $C$ alone, which is unchanged for us, and comparing it to the given hypothesis that the overall error is bounded by $\delta$. Thus the same lower bound statements employed there also hold for us as they only consider the ignored variables. See the second half of Page $8$ in \cite{gupte2021finegrained}.

Finally, the runtime bound required is at least $\Omega(N^{k-o(1)})$ by Theorem \ref{thm:gupte_min_weight}. Note that the final value of $M$ is equal to $M = \binom{N}2 + N + k^2 \le 2N^2$. Therefore, the runtime is at least $\Omega(N^{k-o(1)}) = \Omega(M^{k/2+o(1)})$.
\end{proof}

\section{Upper Bounds}
In this section, we present upper bounds for sparse linear regression and robust linear regression. 
We utilize data structures for the following formulation of the $c$-approximate nearest neighbor problem:

\begin{problem}[Approximate nearest neighbor]
Given a set $P$ of $n$ points in a $d$-dimensional Euclidean space, the $c$-approximate nearest neighbor ($c$-ANN) problem seeks to construct a data structure that, on input query point $q$, reports any point $x\in P$ such that $\|x-q\|_2\le c\,\min_{p\in P}\|p-q\|_2$. 
\end{problem}
In particular, we use the following data structure:
\begin{theorem}
\cite{AndoniR15}
\label{thm:cann}
For any fixed constant $c>1$, there exists a data structure that solves the $c$-ANN problem in $d$-dimensional Euclidean space on $n$ points with $O(dn^{\rho+o(1)})$ query time, $O(n^{1+\rho+o(1)}+dn)$ space, and $O(dn^{1+\rho+o(1)})$ pre-processing time, where $\rho=\frac{1}{2c^2-1}$. 
\end{theorem}
We first get an algorithm, i.e., Algorithm~\ref{alg:compressed_sensing_ub}, for compressed sensing in the noise-less setting. 

\begin{algorithm}[!htb]
\caption{Compressed Sensing Upper Bound \label{alg:compressed_sensing_ub} }
\begin{algorithmic}[1]
\State{\textbf{Input:} Matrix $A \in \R^{n \times d}$, vector $b \in \R^d$, sparsity $k$, accuracy $\eps$} 
\Procedure{CompressedSensing-UpperBound}{$A,b,k,\eps$}
\State $\delta \gets \eps/((2c+2) $

\State $S \gets \emptyset$
\For{each choice $T$ of $k/2$ columns of $A$}
\State{$\mathcal{N} \gets$ $\delta$ net over the image of $A_T$}
\Comment{$A_T$ denotes the submatrix of $A$ with only columns in $T$}\For{each $b' = Ay \in \mathcal{N} $}
\State $S \gets S \cup {(b',y, T)} $
\EndFor
\EndFor
\State $\mathcal{D} \gets c$-approximate nearest neighbor data structure for $\{ b' \mid (b',y, T) \in S\}$
\State Best $ \gets \infty$
\For{each $(b',y,T) \in S$}
\State $b'' \gets b-b'$
\State $(\tilde{b}, \tilde{y}, \tilde{T}) \gets $ output of $\mathcal{D}$ on query $b''$
\State Best $\gets \min(\text{Best}, \|b'+b'' - b\|_2)$ 
\Comment{We are extending $y, \tilde{y}$ to $k/2$ sparse vectors in $\R^d$ in the natural way, i.e., supported on the coordinates of $T$ and $\tilde{T}$ respectively}
\If{Best is updated}
\State Associate vector $z = y+\tilde{y}$ with Best
\EndIf
\EndFor
\State \textbf{Return} the vector $z$ associated with the variable Best
\EndProcedure
\end{algorithmic}
\end{algorithm}

\begin{theorem}\label{thm:compressed_sensing_ub}
Suppose we are given $A \in \R^{n \times d}$ and $b \in \R^n$ such that there exists a $k$-sparse vector $x$ satisfying $Ax = b$. Algorithm \ref{alg:compressed_sensing_ub} returns a $k$-sparse $z$ satisfying $\|Ax-Az\|_2 \le \eps$ in time 
\[ \min_{c \ge 1} O\left( nk \cdot \left( \frac{12c \cdot d \cdot \|b\|_2}{\eps}  \right)^{\frac{k}{2} \cdot (1+ 1/(2c^2-1))} \right). \]
\end{theorem}
\begin{proof}
By scaling, we can assume $b$ to be a unit vector without loss of generality. This means our $\eps$ term will implicitly be multiplied by a $\|b\|_2$ factor. We first show that $z$ satisfies the approximation guarantees of Theorem \ref{thm:compressed_sensing_ub}. Note $z$ is $k$-sparse since all $y$ and $\tilde{y}$ considered in line $18$ of Algorithm \ref{alg:compressed_sensing_ub} are $k/2$-sparse in $\R^d$ (all $y's$ considered satisfy that $Ay$ is in the image of some $k/2$ columns of $A$). Our task now is to show $\|Az-Ax\|_2 \le \eps$.
Now let $x = x_1 + x_2$ be any division of $x$ into the sum of two $k/2$-sparse vectors $x_1$ and $x_2$. 
Define $b_1 = Ax_1, b_2 = Ax_2$ so that $b=b_1+b_2$ and let $y$ be such that $Ay$ is the closest point in some net $\mathcal{N}$ (considered in line $6$ of the algorithm) to $b_1$. By the choice of our net, we know that $\|Ay - Ax_1\|_2 \le \delta$. 
We have
\begin{equation}\label{eq:cs_up_proof_1}
   \|(b-Ay) - b_2 \|_2 = \|(b-Ay) - (b-Ax_1)\|_2 = \|Ay - Ax_1\|_2 \le \delta.  
\end{equation}
Letting $b' = Ay$ for this $y$ and considering this choice of $y$ and $b'$ in the loop on line $15$ of Algorithm \ref{alg:compressed_sensing_ub}, the above calculation shows the existence of a $b'' = b-b'$ such that $\|b_2-b''\| \le \delta$. By the construction of $\mathcal{D}$ in line $11$ of Algorithm \ref{alg:compressed_sensing_ub}, it follows that $\mathcal{D}$ will output a $\tilde{b}$ on query $b''$ such that 
\begin{equation}\label{eq:cs_up_proof_2}
    \|b'' - \tilde{b} \|_2 \le 2c\delta.
\end{equation}
This is because there must exist some $\bar{y} \in \mathcal{N}$ such that $\|Ax_2 - A\bar{y}\|_2 \le \delta$ by our net.
Thus since $\|b_2-b''\| \le \delta$, we have $\| b'' - \bar{b}\|_2 \le 2 \delta$ by triangle inequality and \eqref{eq:cs_up_proof_2} follows since $\mathcal{D}$ is only guaranteed to return a $c$-approximate neighbor. 

Altogether, for this $y$ and corresponding $\tilde{y}$ from line $17$ of Algorithm \ref{alg:compressed_sensing_ub}, we have 
\begin{align*}
    \|A(y+\tilde{y}) - b\|_2 &= \|A(y+\tilde{y}) - (Ax_1 + Ax_2)\|_2 \\
    &= \|(Ay-Ax_1) + (A\tilde{y} - Ax_2)\|_2 \\
    &\le \|A(y-x_1)\|_2 + \|A\tilde{y} - b_2\|_2 \\
    &\le \delta + \|A\tilde{y} - b_2\|_2 \qquad \text{(from \eqref{eq:cs_up_proof_1})}\\
    &\le \delta + \|\tilde{b} - b''\|_2 + \|b'' - b_2\|_2 \qquad \text{(by triangle inequality and $A\tilde{y}=\tilde{b}$)}\\
    &\le\delta + 2c \delta + \delta \qquad \text{(from \eqref{eq:cs_up_proof_2})}\\
    &= (2c+2)\delta \\
    &\le \eps,
\end{align*}
or in other words, $\|A(x-z)\|_2 \le \eps$, as desired.

We now compute the runtime of Algorithm \ref{alg:compressed_sensing_ub}. The size of $\mathcal{N}$ is upper bounded by 
\[ |\mathcal{N}| \le \left(\frac{3\cdot  }{\delta} \right)^{k/2} \le \left(\frac{12c}{\eps} \right)^{k/2}.\]
Picking all the choice of $T$ and looping over all the vectors in $\mathcal{N}$ in the for loop of line $5$ of Algorithm \ref{alg:compressed_sensing_ub} takes time at most $O(d^{k/2} \cdot |\mathcal{N}| \cdot nk)$.
Initializing $\mathcal{D}$ on the set $S$ is dominated by the $|S|$ many queries performed on $\mathcal{D}$, which we discuss now. Looping over $S$ and querying $\mathcal{D}$ in the for loop on line $15$ takes time $O(|S| \cdot k \cdot |S|^{1/(2c^2-1)})$ by the guarantees of approximate nearest neighbor search~\cite{AndoniR15} in Theorem~\ref{thm:cann}. Note that $|S| \le d^{k/2} \cdot |\mathcal{N}|$. Putting everything together, the total runtime is
\[O\left( d^{k/2} \cdot |\mathcal{N}| \cdot nk +  k \cdot |S| \cdot |S|^{1/(2c^2-1)}\right) = O\left( nk \cdot \left( \frac{12  c \cdot d }{\eps}  \right)^{\frac{k}{2} \cdot (1+ 1/(2c^2-1))} \right). \qedhere\]
\end{proof}

By modifying the size of the net, we can also find a sparse vector $z$ that is arbitrarily close to the sparse vector $x$ as well.

\begin{corollary}\label{cor:compressed_sensing_corollary}
Suppose we are given $A \in \R^{n \times d}$ and $b \in \R^n$ such that there exists a $k$-sparse vector $x$ satisfying $Ax = b$. 
There exists an algorithm that returns a $k$-sparse $z$ satisfying $\|z-x\|_2 \le \eps$ in time 
\[ \min_{c \ge 1} O\left( nk \cdot \left( \frac{12c \cdot d \cdot \kappa \cdot \|b\|_2}{\eps}  \right)^{\frac{k}{2} \cdot (1+ 1/(2c^2-1))} \right) \]
where $\kappa = \sigma_{\text{max}}(A)/\sigma_{\text{min}}(A)$ where the ratio is over non-zero singular values.
\end{corollary}
\begin{proof}
We modify Algorithm~\ref{alg:compressed_sensing_ub} by setting $\eps' = \eps/\kappa$ in Theorem \ref{thm:compressed_sensing_ub}. This implies $\|Az - Ax\|_2 \le \eps/\kappa$. Letting $A^{+}$ denote the pseudo-inverse of $A$, we have that
\[\|x-z\|_2 = \|A^{+}A(x-z)\|_2 \le \|A^{+}\|_2 \|A(x-z)\|_2 \le \kappa \cdot \frac{\eps}\kappa \le \eps, \]
as desired. The runtime follows from the parameter change.
\end{proof}

\begin{remark}\label{rem:choices_of_params}
By letting $c = 2$ and under the assumption $\|b\|_2 = O(1)$, we achieve the runtime of $nk \cdot ( O(d)/\eps)^{k/2 \cdot (1 + 1/7)}$. By letting $c = \Theta(1/\sqrt{\eps})$, we can achieve the runtime of $nk \cdot ( O(d)/\eps^{1.5})^{k/2 \cdot (1 + \eps)}$. In general, the best choice for $c$ depends on the relationship between $d$ and $k$.
\end{remark}

\begin{remark}
Note that in Corollary \ref{cor:compressed_sensing_corollary}, we can replace the parameter $\kappa$ by the largest condition number $\kappa'$ of any $n \times k$ submatrix of $A$ since $x-z$ is a $k$-sparse vector. This is an improvement as $\kappa' \le \kappa$.
\end{remark}

Now consider the setting where $x$ is a binary signal or each coordinate has a finite number of choices in general. This setting is motivated by a number of applications including wideband spectrum sensing, wireless networks, group testing, error correcting codes, spectrum hole detection for cognitive radios, massive Multiple-Input Multiple-Output (MIMO) channels, etc. to name a few \cite{Cands2005ErrorCV, Rossi2014SpatialCS, Axell2012SpectrumSF, Dymarski2013SparseSM, Ens2013OptimizedSW, Keiper2016CompressedSF, Fosson2018NonconvexLA, Karahanoglu2012AOM, Nakarmi2012BCSCS, Meng2010CollaborativeSS}.

In this setting, we can \emph{recover} $x$ perfectly in roughly $O(d)^{k/2}$ time using Algorithm \ref{alg:compressed_sensing_ub_finite}.

\begin{algorithm}[!htb]
\caption{Compressed Sensing Upper Bound for Finite Valued Signals\label{alg:compressed_sensing_ub_finite} }
\begin{algorithmic}[1]
\State{\textbf{Input:} Matrix $A \in \R^{n \times d}$, vector $b \in \R^d$, sparsity $k$, accuracy $\eps$} 
\Procedure{CompressedSensing-UpperBound-Finite}{$A,b,k,\eps$}

\State $S \gets \emptyset$
\For{each choice $T$ of $k/2$ columns of $A$}
\For{each $w = (w_i)_{i \in T} \in L^{k/2}$}
\State $b' \gets \sum_{i \in T}w_iA_{*,i}$
\Comment{$A_{*,i}$ denotes the $i$th column of $A$}
\State $S \gets S \cup {(b', T, w)} $
\EndFor
\EndFor
\State $\mathcal{D} \gets $ hashtable for the values $\{b'| (b',T, w) \in S\}$
\State Best $ \gets \infty$
\For{each $(b',T, w) \in S$}
\State $b'' \gets b-b'$
\State $(b'', T'', w'') \gets $ output of $\mathcal{D}$ on query $b''$
\Comment{If $b''$ is not in hash table, continue for loop on step 12}
\State Best $\gets \min(\text{Best}, \|\sum_{i \in T}w_iA_{*,i} + \sum_{i \in T''}w''_iA_{*,i} - b\|_2)$ 
\If{Best is updated}
\State Associate vector $z = \sum_{i \in T}w_iA_{*,i} + \sum_{i \in T''}w''_iA_{*,i}$ with Best
\EndIf
\EndFor
\State \textbf{Return} the vector $z$ associated with the variable Best
\EndProcedure
\end{algorithmic}
\end{algorithm}

\begin{theorem}\label{thm:compressed_sensing_ub_finite}
Suppose we are given $A \in \R^{n \times d}$ and $b \in \R^n$ such that there exists a unique $k$-sparse vector $x$ satisfying $Ax = b$. Furthermore, suppose that each coordinate of $x$ must lie in the set $\{0\} \cup U$. Algorithm \ref{alg:compressed_sensing_ub_finite} recovers $x$ exactly in time
\[nk (d|U|)^{k/2} + O(1) \cdot (d|U|)^{k/2}. \]
\end{theorem}
\begin{proof}
First we prove the approximation guarantee. Consider $x = x_1 + x_2$ be any decomposition of $x$ into sum of two $k/2$-sparse vectors. Consider the for loop in step $10$ of Algorithm \ref{alg:compressed_sensing_ub_finite}. We must consider $b' = Ax_1$ at some iteration since we looped over all possible $k/2$ sparse vectors in step $4$ of Algorithm \ref{alg:compressed_sensing_ub_finite}. For this choice of $b'$, we have that $b'' = b-b'$ satisfies $Ax_2 = b''$. Therefore, the hash table $\mathcal{D}$ will return $x_2$ on query $b''$. Finally since $x_1$ and $x_2$ have disjoint support, the vector $z = \sum_{i \in T}w_iA_{*,i} + \sum_{i \in T''}w''_iA_{*,i}$ is indeed equal to $x$ as desired.

We now analyze the runtime. Forming the set $S$ takes time at most $nk (d|U|)^{k/2}$ time and querying the hash table takes $O(1)$ time each in expectation. Thus the overall runtime is $nk (d|U|)^{k/2} + O(1) \cdot (d|U|)^{k/2}$.
\end{proof}

\begin{remark}
In the case of binary vectors, $U = \{1\}$ so we achieve the runtime of $(nk + O(1)) \cdot d^{k/2}$.
\end{remark}

Using Corollary \ref{cor:general_reduction}, we can also get an algorithm for robust regression, formalizing the conditions of Theorem~\ref{thm:robust_ub:informal}. 

\begin{theorem}\label{thm:robust_ub}
Suppose we are given $A \in \R^{n \times d}$ and $b \in \R^n$ such that \[\min_{S,x} \|S(Ax-b)\|_2 = 0 \]
where $S$ is constrained to be a diagonal matrix with $n-k$ one entries in the diagonal and $k$ zero entries. Algorithm  \ref{alg:robust_ub} returns a diagonal matrix $S'$ with $n-k$ one entries in the diagonal and $k$ zero entries and a $x'$ such that 
$\|S'(Ax'-b)\|_2 \le \eps$ in time 
\[ \min_{c \ge 1} O\left( nk \cdot \left( \frac{12c \cdot n \cdot \|b\|_2 }{\eps}  \right)^{\frac{k}{2} \cdot (1+ 1/(2c^2-1))} \right). \]
\end{theorem}

\begin{algorithm}[!htb]
\caption{Robust Regression Upper Bound\label{alg:robust_ub} }
\begin{algorithmic}[1]
\State{\textbf{Input:} Matrix $A \in \R^{n \times d}$, vector $b \in \R^d$, sparsity $k$, accuracy $\eps$} 
\Procedure{RobustRegression-UpperBound}{$A,b,k,\eps$}
\If{Columns of $A$ span all of $\R^n$}
\State Let $x'$ be such that $Ax'=b$
\State \textbf{Return} any diagonal matrix $S'$ with $n-k$ one entries on diagonal and $x'$
\EndIf
\State Let $X$ be such that $XA = 0$ with orthonormal rows
\State $c \gets -Xb$
\State $z \gets $ output of Algorithm \ref{alg:compressed_sensing_ub} on input $(X, c, k, \eps)$
\State $x' \gets $ $\argmin_{x'} \|Ax'-b -z\|_2$
\State $S' \gets$ diagonal matrix with diagonal entry encoding the zero coordinates of $z$
\Comment{Note $z$ has $k$ non zero coordinates and thus, $S'$ has $n-k$ ones on the diagonal and $k$ zeros }
\State \textbf{Return} $S'$ and $x'$
\EndProcedure
\end{algorithmic}
\end{algorithm}

\begin{proof}
We assume we are not in the if statement  case in line $3$ of Algorithm \ref{alg:robust_ub} since otherwise we are done. Consider the quantities $X$ and $c$ associated with the robust regression instance arising from Corollary \ref{cor:general_reduction}. Now from Corollary \ref{cor:general_reduction}, we know the following statements: first, we know that all non zero singular values of $X$ must be $1$ since $XX^T$ is the identity matrix and the non zero eigenvalues of $X^TX$ and $XX^T$ are the same. Second, we know that $\|z-w\|_2 \le \eps$ where $Xw = c$. This implies that if $Xz = c'$ then $\|c-c'\| \le \eps$. Therefore, we have $Xz = c + e$ or in other words, $X(z+b) = e$ for $\|e\|_2 \le \eps$. Now write $z+b = v_1 + v_2$ where $v_1$ is the orthogonal projection of $z+b$ onto the column span of $A$. We know that $Xv_1 = 0$ by definition and thus, there exists a $\tilde{x}$ such that $A\tilde{x} = v_1$. This $\tilde{x}$ satisfies 
\[\|A\tilde{x} - (z+b)\|_2 = \|v_2\|_2 = \|e\|_2 \le \eps.\]
Therefore, we know that $x'$ chosen in line $10$ of Algorithm \ref{alg:robust_ub} satisfies $\|Ax'-b-z\|_2 \le \eps$ as well. Finally, $z$ is a $k$-sparse vector since it is the output of Algorithm \ref{alg:compressed_sensing_ub} which implies that our choice of $S'$ in line $11$ of Algorithm \ref{alg:robust_ub} gives us 
\[\|S'(Ax'-b)\|_2 \le \eps \]
as desired.

We now compute the runtime of Algorithm \ref{alg:robust_ub}. The runtime is dominated by the call to Algorithm \ref{alg:compressed_sensing_ub}. Note that our matrix $X$ has all non-zero singular values equal to $1$ so we can take $\kappa$ in Theorem \ref{thm:compressed_sensing_ub} (more specifically in Corollary \ref{cor:compressed_sensing_corollary}) to be equal to $1$. Furthermore, we can check that $\|c\|_2 \le \|b\|_2$.
\end{proof}
\begin{remark}
The same comments as in Remark \ref{rem:choices_of_params} apply, except we no longer have a $\kappa$ dependency and therefore do not need any assumptions on its value.
\end{remark}

\section{Sparse PCA}
In this section we present our results for the sparse PCA problem which is defined as follows. Let $A$ be a PSD matrix of rank $r$. The goal of sparse PCA is to solve 
\begin{equation}\label{eq:sparse_pca}
    \max_{\|v\|_2 = 1, \,  \|v\|_0 \le k} v^TAv. 
\end{equation}

Sparse PCA is known to be NP-hard to solve exactly and approximate with a $1-\delta$ factor for some small constant $\delta > 0$ \cite{chan2016}. It is also known how to obtain a $k$-sparse unit vector $v$ which achieves at least a $1-\eps$ approximation to the objective in time  $O((4/\eps)^{r} \cdot n \cdot k^2)$ \cite{AsterisPKD15}.

We obtain an algorithm with an improved dependence on the exponent of $1/\eps$ via a novel connection to the computational geometry and our ideas from the prior sections. We need the following theorems from \cite{chan2018} about the runtime of approximating the diameter and the bichromatic farthest pair of point sets. 

\begin{theorem}\label{thm:diam_ref}
Consider $P \subset \R^d$ with $|P| = n$. Let 
\[\textup{diam}(P) = \max_{x,y \in P} \|x-y\|_2\]
denote the diameter of $P$.
There exists an algorithm which computes $x',y'$ such that 
\[\|x'-y'\|_2 \ge (1-\eps) \, \textup{diam}(P)\] in time $\widetilde{O}(nd/\sqrt{\eps} + (1/\eps)^{d/2 + 1})$.
\end{theorem}

\begin{theorem}\label{thm:BFP_ref}
Consider $P, Q \subset \R^d$ with $|P \cup Q| = n$. There exists an algorithm which computes $x' \in P$ and $y' \in Q$ such that 
\[ \|x'-y'\|_2 \ge (1-\eps) \max_{x \in P, y \in Q} \|x-y\| \]
 in time $\widetilde{O}(nd/\sqrt{\eps} + (1/\eps)^{d/2 + 1})$.
\end{theorem}

We now show Theorem~\ref{thm:sparse_pca_ub_exact:informal}. The main idea behind our algorithm is to note that since $A = B^TB$, we have $v^TAv = \|Bv\|_2^2$. Now we employ a trick from our previous algorithmic result: we partition $v$ into the difference of two $k/2$ sparse vectors $x_1 - x_2$. This gives us $v^TAv = \|Bx_1 - Bx_2\|_2^2$. Thus, maximizing the original quadratic form over $A$ reduces to maximizing the distance between the points $\{Bx_1\}$ and $\{Bx_2\}$ where $x_1$ and $x_2$ range over all $k/2$ sparse vectors. To do this, we invoke the \texttt{Bichromatic-Farthest-Pair} guarantee from \cite{chan2018} to find the best pair of $k/2$ sparse vectors, say $y_1, -y_2$. 

One technical issue is we must ensure the supports of $y_1$ and $y_2$ are \emph{disjoint}. Since we have to maximize over unit norm vectors $v$, if the supports of $y_1$ and $y_2$ overlap then we cannot say anything about the norm of $y_1-y_2$. To get around this, we first randomly partition $[n]$ into two disjoint sets and only search over $k/2$ sparse $y_1$ with support completely contained in the first set and vice versa for $y_2$. We formalize the argument below.

\begin{theorem}\label{thm:sparse_pca_ub_exact}
Algorithm \ref{alg:sparse_pca_ub_exact} returns a unit vector $u$ satisfying $\|u\|_0 \le k$ such that
\[u^TAu \ge (1-\eps) \max_{\|v\|_2 = 1 \,,  \|v\|_0 \le k} v^TAv \] with probability $1 - \exp(-\Omega( k \eps^2))$. The runtime is 
\[\widetilde{O}\left(\frac{k^2}\eps \cdot \left( r\left( \frac{n\kappa}{\eps} \right)^{k(1+\eps)/2} + \left( \frac{1}{\eps} \right)^{r/2 + 1} \right) \right)\]
where $\kappa$ is the ratio of the largest and smallest non-zero singular values of $A$.
\end{theorem}

\begin{proof}
We first prove the approximation guarantee. Let $v$ be the optimum $k$-sparse unit vector. Note that with probability $1 - \exp(-\Omega( k \eps^2))$, via a Chernoff bound, both $U_1$ and $U_2$ contain at least $k(1-\eps)/2$ number of support indices of $v$. We now condition in this event. Let $v = x_1 - x_2$ be a decomposition of $v$ into the difference of two vectors such that the support of $x_1$ is contained entirely in $U_1$ and similarly, the support of $x_2$ is contained entirely in $U_2$. 
We know that we will loop over some $z \in \mathcal{N}'$ in step $9$ and some $T_1$ in step $12$ of Algorithm \ref{alg:sparse_pca_ub_exact} which satisfies $| \|x_1\|_2^2 - z | \le \eps/2$ and $T$ exactly encodes the support of $x_1$. 
Similarly, we will loop over some $T_2$ in step $20$ of Algorithm \ref{alg:sparse_pca_ub_exact} such that $T_2$ exactly encodes the support of $x_2$. Now let $y_1$ and $y_2$ satisfy the following:
\begin{itemize}
    \item $|\|y_1\|_2^2 - \|x_1\|_2^2| \le \eps/2$, $|\|y_2\|_2^2 - \|x_2\|_2^2| \le \eps/2$,
    \item the supports of $y_1$ and $y_2$ exactly match those of $x_1$ and $x_2$ respectively, and
    \item $\|y_1 - x_1\|_2 \le \delta$, $\|y_2 - x_2\|_2 \le \delta$.
\end{itemize}
Such a $y_1, y_2$ exist because of the net $\mathcal{N}$ and we looped over all choices $k_1, k_2$ such that $k_1 + k_2 = k$. We know that $w = y_1-y_2$ satisfies 
\[\|w\|_2^2=\|y_1\|_2^2+\|y_2\|_2^2\ge1-\eps,\]
since $y_1, y_2$ have disjoint support. Furthermore, we have 
\[\|B(x_1 - y_1)\| \le \eps/\kappa  \]
and 
\[\|B(x_2 - y_2)\| \le \eps/\kappa  \]
by our choice of $\delta$ in the net $\mathcal{N}$.
Furthermore,
\[v^TAv = v^TB^TBv = \|Bv\|_2^2 = \|Bx_1 - Bx_2\|_2^2  \]
so by the above calculations, 
\[ w^TAw = \|By_1 - By_2\|_2^2 \ge (1-\eps)\|Bx_1 - Bx_2\|_2^2  = v^TAv.\]
The guarantees of \texttt{Bichromatic-Farthest-Pair} imply that we find a $w'$ such that 
\[w'^TAw' \ge (1-\eps)w^TAw \ge (1-O(\eps))v^TAv \]
in step $29$ of  Algorithm \ref{alg:sparse_pca_ub_exact}. Furthermore, $w'$ has norm at least $1-\eps$ by our requirements on $y_1$ and $y_2$ in Algorithm \ref{alg:sparse_pca_ub_exact} so we get the desired approximation.

We now analyze the runtime. The runtime is consists of guessing over $k_1$ a $k_2$ with is $O(k^2)$ time and guessing over $z$ which has $O(1/\eps)$ choices. Looping over the choices of $k_1$ columns (or $k_2$ columns) is time $O(n^{k(1+\eps)/2}(\kappa/\eps)^{k(1+\eps)/2})$. The total number of points in the each instance of \texttt{Bichromatic-Farthest-Pair} used is $O(n^{k(1+\eps)/2}(\kappa/\eps)^{k(1+\eps)/2})$ and all vectors in the instance are in dimension $r$. Invoking Theorem \ref{thm:BFP_ref}, the overall runtime is $\widetilde{O}(k^2/\eps \cdot ( r(n\kappa/\eps)^{k(1+\eps)/2} + (1/\eps)^{r/2 + 1}))$.
\end{proof}

\begin{remark}
Note that we can think of the parameter $k$ as much smaller than $r$. Thus the dominant term in our runtime is $(1/\eps)^{r/2+1}$ which improves upon the dependence of  $(4/\eps)^r$ as $ (4/\eps)^{r} \gg (1/\eps)^{r/2+1} $.
\end{remark}

\begin{algorithm}[!htb]
\caption{Sparse PCA Upper Bound\label{alg:sparse_pca_ub_exact} }
\begin{algorithmic}[1]
\State{\textbf{Input:} PSD Matrix $A \in \R^{n \times n}$ of rank $r$, sparsity $k$, accuracy $\eps$} 
\Procedure{SparsePCA-UpperBound}{$A,k,\eps$}
\State Compute $B$ such that $B^TB \gets A$
\Comment{$B$ is a $r \times n$ matrix}
\State $\kappa \gets \sigma_{\text{max}}(A)/\sigma_{\text{min}}(A)$
\Comment{The $\min/\max$ is over non-zero singular values}
\State $\delta \gets \eps/\kappa$

\State $U_1 \cup U_2 \gets$ random partition of $[n]$ into two disjoint subsets 
\State $\mathcal{N}' \gets \eps/2$-net of the unit interval $[0,1]$

\For{each choice of $k_1, k_2$ such that $k_1 + k_2 = k$ and $k_1, k_2 \ge k(1-\eps)/2$}
\For{each choice of $z \in \mathcal{N}'$}
\State $S_1 \gets \emptyset$
\State $S_2 \gets \emptyset$

\For{every choice $T$ of $k_1$ columns of $B$ restricted to the indices in $U_1$}
\Comment{$B_T$ is $r \times k_1$ matrix restricted to columns in $T$}
\State $\mathcal{N} \gets \delta$-net of ball of radius $1$ in $\R^{k_1}$
\For{each $y \in \mathcal{N}$}
\If{$\|y\|_2^2 \in [z- \eps/2, z+\eps/2]$}
\State{$S_1 \gets S_1 \cap \{B_Ty\}$}
\EndIf
\EndFor
\EndFor

\For{every choice $T$ of $k_2$ columns of $B$ restricted to the indices in $U_1$}
\Comment{$B_T$ is $r \times k_2$ matrix restricted to columns in $T$}
\State $\mathcal{N} \gets \delta$-net of ball of radius $1$ in $\R^{k_2}$
\For{each $y \in \mathcal{N}$}
\If{$\|y\|_2^2 \in [1-z- \eps/2, 1-z+\eps/2]$}
\State{$S_2 \gets S_2 \cap \{-B_Ty\}$}
\EndIf
\EndFor
\EndFor

\State{ $ (B_{T_1}y_1, -B_{T_2}y_2) \gets$ solution of \texttt{Bichromatic-Farthest-Pair} on the pair $(S_1, S_2)$}
\State{$w = (y_1 - y_2)/\|y_1 - y_1\|_2$}
\State{Keep track of the maximum value of $w^TAw$ encountered }
\EndFor
\EndFor
\State{\textbf{Return} $w$ that maximizes $w^TAw$ over all $w$'s observed}

\EndProcedure
\end{algorithmic}
\end{algorithm}

\subsection{Sparse PCA with Limited Alphabet}

In this section, we consider a slight variation of \texttt{SparsePCA} where the entries in $v$ are limited to a small alphabet. Formally, we consider the problem of 
\[\max_{\forall i: v_i \in \{-L, \ldots, L\}, \, \|v\|_0 \le k} v^TAv  \]
where $A$ is a $n \times n$ PSD matrix of rank $r$. This formulation is motivated by its connection the the Densest $k$-Subgraph problem where we wish to maximize $v^TAv$ over vectors $v$ with a limited range of choices per coordinate but the matrix $A$ is not necessarily PSD which holds in our case; see \cite{ChanPR16} for more information about the Densest $k$-Subgraph problem.

For a relaxation of this version, we can also obtain an algorithm with an exponentially better dependence on $\eps$ than the result from \cite{AsterisPKD15} via a novel connection to the computational geometry problem of \texttt{Diameter}.

\begin{theorem}\label{thm:sparse_pca_ub}
Algorithm \ref{alg:sparse_pca_ub} returns a $u$ satisfying $\|u\|_0 \le k$ such that
\[u^TAu \ge (1-\eps) \max_{\forall i: v_i \in \{-L, \ldots, L\} \,,  \|v\|_0 \le k} v^TAv \]  and  all the entries of $u$ are in the set $\{-2L , \ldots, 2L\}$ in time $\widetilde{O}((1/\eps)^{r/2 + 1} + rk \cdot ((2L+1)n)^{k/2}/\sqrt{\eps})$.
\end{theorem}

\begin{algorithm}[!htb]
\caption{Sparse PCA Limited Alphabet Upper Bound\label{alg:sparse_pca_ub} }
\begin{algorithmic}[1]
\State{\textbf{Input:} PSD Matrix $A \in \R^{n \times n}$ of rank $r$, sparsity $k$, accuracy $\eps$} 
\Procedure{SparsePCA-UpperBound-Limited}{$A,k,\eps$}

\State Compute $B$ such that $B^TB \gets A$
\Comment{$B$ is a $r \times n$ matrix}
\State $S \gets \emptyset$
\For{every choice $T$ of $k/2$ columns of $B$}
\Comment{$B_T$ is $r \times k/2$ matrix restricted to columns in $T$}
\For{each $w = (w_i)_{i \in T} \in \{-L, \ldots, L\}^{k/2}$}
\State $y \gets \sum_{i \in T} w_i A_{*, i}$
\Comment{$A_{*,i}$ denotes the $i$th column of $A$}
\State $S \gets S \cup \{ B_Ty \}$
\EndFor
\EndFor
\State $(By,By') \gets $ solution of \texttt{Diameter} on $S$ using Theorem \ref{thm:diam_ref}
\State Return $u = y - y'$
\EndProcedure
\end{algorithmic}
\end{algorithm}

\begin{proof}[Proof of Theorem \ref{thm:sparse_pca_ub}]
We prove correctness first. Consider the optimal $v$ and let $v = x_1 + x_2$ for $k/2$ sparse vectors $x_1, x_2$. Note that 
\begin{align*}
    \textup{OPT} &= v^TAv = v^TB^TBv \\
    &= x_1^TB^TBx_1 + 2x_1B^TBx_2 + x_2^TB^TBx_2 \\
    &= \|Bx_1\|_2^2 + \|Bx_2\|_2^2 + 2\langle Bx_1, Bx_2 \rangle.
\end{align*}
Letting $x_2' = -x_2$, we get
\[v^TAv = \|Bx_1 - Bx_2'\|_2^2. \]
Now consider $y,y'$ returned by Algorithm \ref{alg:sparse_pca_ub}. From the guarantees of \texttt{Diameter}, it follows that 
\[ u^TBu = \|By - By'\|_2^2 \ge (1-\eps) \|Bx_1 - Bx_2'\|_2^2 = (1-\eps)\textup{OPT}.\]
Finally, $u$ is $k$ sparse and its entries belong to $\{-2L, \cdots, 2L\}$. The runtime follows from Theorem \ref{thm:diam_ref} using the fact that $|S| = ((2L+1)d)^{k/2}$.
\end{proof}

\section*{Acknowledgments} 
Eric Price is supported by NSF awards CCF-2008868, CCF-1751040 (CAREER), and NSF IFML 2019844. 
Sandeep Silwal is supported by an NSF Graduate Research Fellowship under Grant No.\ 1745302, NSF TRIPODS program (award DMS-2022448), NSF award CCF-2006798, and Simons Investigator Award. 
Samson Zhou is supported by a Simons Investigator Award of David P. Woodruff.

\bibliographystyle{alpha}
\bibliography{bib}

\newcommand{\etalchar}[1]{$^{#1}$}
\begin{thebibliography}{TWZ{\etalchar{+}}22}

\bibitem[ALLP12]{Axell2012SpectrumSF}
Erik Axell, Geert Leus, Erik~G. Larsson, and H.~Vincent Poor.
\newblock Spectrum sensing for cognitive radio : State-of-the-art and recent
  advances.
\newblock {\em IEEE Signal Processing Magazine}, 29:101--116, 2012.

\bibitem[APKD15]{AsterisPKD15}
Megasthenis Asteris, Dimitris~S. Papailiopoulos, Anastasios Kyrillidis, and
  Alexandros~G. Dimakis.
\newblock Sparse {PCA} via bipartite matchings.
\newblock In {\em Advances in Neural Information Processing Systems 28: Annual
  Conference on Neural Information Processing Systems}, pages 766--774, 2015.

\bibitem[AR15]{AndoniR15}
Alexandr Andoni and Ilya~P. Razenshteyn.
\newblock Optimal data-dependent hashing for approximate near neighbors.
\newblock In {\em Proceedings of the Forty-Seventh Annual {ACM} on Symposium on
  Theory of Computing, {STOC}}, pages 793--801, 2015.

\bibitem[AW08]{amini2008high}
Arash~A. Amini and Martin~J. Wainwright.
\newblock High-dimensional analysis of semidefinite relaxations for sparse
  principal components.
\newblock In {\em 2008 IEEE international symposium on information theory},
  pages 2454--2458, 2008.

\bibitem[BGM20]{broderick2020automatic}
Tamara Broderick, Ryan Giordano, and Rachael Meager.
\newblock An automatic finite-sample robustness metric: Can dropping a little
  data change conclusions.
\newblock {\em arXiv preprint arXiv:2011.14999}, page~16, 2020.

\bibitem[BJK15]{BhatiaJK15}
Kush Bhatia, Prateek Jain, and Purushottam Kar.
\newblock Robust regression via hard thresholding.
\newblock In {\em Advances in Neural Information Processing Systems 28: Annual
  Conference on Neural Information Processing Systems}, pages 721--729, 2015.

\bibitem[BJKK17]{Bhatia0KK17}
Kush Bhatia, Prateek Jain, Parameswaran Kamalaruban, and Purushottam Kar.
\newblock Consistent robust regression.
\newblock In {\em Advances in Neural Information Processing Systems 30: Annual
  Conference on Neural Information Processing Systems}, pages 2110--2119, 2017.

\bibitem[BP21]{BakshiP21}
Ainesh Bakshi and Adarsh Prasad.
\newblock Robust linear regression: optimal rates in polynomial time.
\newblock In {\em {STOC} '21: 53rd Annual {ACM} {SIGACT} Symposium on Theory of
  Computing}, pages 102--115, 2021.

\bibitem[CAT{\etalchar{+}}20]{Cherapanamjeri2020}
Yeshwanth Cherapanamjeri, Efe Aras, Nilesh Tripuraneni, Michael~I. Jordan,
  Nicolas Flammarion, and Peter~L. Bartlett.
\newblock Optimal robust linear regression in nearly linear time.
\newblock {\em CoRR}, abs/2007.08137, 2020.

\bibitem[CBZ{\etalchar{+}}22]{ChowdhuryBZWD22}
Agniva Chowdhury, Aritra Bose, Samson Zhou, David~P. Woodruff, and Petros
  Drineas.
\newblock A fast, provably accurate approximation algorithm for sparse
  principal component analysis reveals human genetic variation across the
  world.
\newblock In {\em Research in Computational Molecular Biology - 26th Annual
  International Conference, {RECOMB}, Proceedings}, pages 86--106, 2022.

\bibitem[CDS98]{ChenDS98}
Scott~Shaobing Chen, David~L. Donoho, and Michael~A. Saunders.
\newblock Atomic decomposition by basis pursuit.
\newblock {\em {SIAM} J. Sci. Comput.}, 20(1):33--61, 1998.

\bibitem[CDWZ20]{ChowdhuryDWZ20}
Agniva Chowdhury, Petros Drineas, David~P. Woodruff, and Samson Zhou.
\newblock Approximation algorithms for sparse principal component analysis.
\newblock {\em CoRR}, abs/2006.12748, 2020.

\bibitem[Cha18]{chan2018}
Timothy~M. Chan.
\newblock Applications of chebyshev polynomials to low-dimensional
  computational geometry.
\newblock {\em J. Comput. Geom.}, 9:3--20, 2018.

\bibitem[CJ95]{cadima1995loading}
Jorge Cadima and Ian~T. Jolliffe.
\newblock Loading and correlations in the interpretation of principle
  compenents.
\newblock {\em Journal of applied Statistics}, 22(2):203--214, 1995.

\bibitem[CP86]{coleman1986null}
Thomas~F Coleman and Alex Pothen.
\newblock The null space problem i. complexity.
\newblock {\em SIAM Journal on Algebraic Discrete Methods}, 7(4):527--537,
  1986.

\bibitem[CPR16a]{ChanPR16}
Siu~On Chan, Dimitris Papailliopoulos, and Aviad Rubinstein.
\newblock On the approximability of sparse {PCA}.
\newblock In {\em Proceedings of the 29th Conference on Learning Theory,
  {COLT}}, pages 623--646, 2016.

\bibitem[CPR16b]{chan2016}
Siu~On Chan, Dimitris Papailliopoulos, and Aviad Rubinstein.
\newblock On the approximability of sparse pca.
\newblock In {\em COLT}, 2016.

\bibitem[CRTV05]{Cands2005ErrorCV}
Emmanuel~J. Cand{\`e}s, Mark Rudelson, Terence Tao, and Roman Vershynin.
\newblock Error correction via linear programming.
\newblock In {\em FOCS 2005}, 2005.

\bibitem[CT07]{candes2007dantzig}
Emmanuel Candes and Terence Tao.
\newblock The dantzig selector: Statistical estimation when p is much larger
  than n.
\newblock {\em The annals of Statistics}, 35(6):2313--2351, 2007.

\bibitem[CW15]{ClarksonW15}
Kenneth~L. Clarkson and David~P. Woodruff.
\newblock Sketching for \emph{M}-estimators: {A} unified approach to robust
  regression.
\newblock In {\em Proceedings of the Twenty-Sixth Annual {ACM-SIAM} Symposium
  on Discrete Algorithms, {SODA}}, pages 921--939, 2015.

\bibitem[CWW19]{ClarksonWW19}
Kenneth~L. Clarkson, Ruosong Wang, and David~P. Woodruff.
\newblock Dimensionality reduction for tukey regression.
\newblock In {\em Proceedings of the 36th International Conference on Machine
  Learning, {ICML}}, 2019.

\bibitem[CYW19]{ChenYW19}
Yichen Chen, Yinyu Ye, and Mengdi Wang.
\newblock Approximation hardness for {A} class of sparse optimization problems.
\newblock {\em J. Mach. Learn. Res.}, 20:38:1--38:27, 2019.

\bibitem[dBEG08]{d2008optimal}
Alexandre d'Aspremont, Francis Bach, and Laurent El~Ghaoui.
\newblock Optimal solutions for sparse principal component analysis.
\newblock {\em Journal of Machine Learning Research}, 9(7), 2008.

\bibitem[dGJL07]{dAspremontGJL07}
Alexandre d'Aspremont, Laurent~El Ghaoui, Michael~I. Jordan, and Gert R.~G.
  Lanckriet.
\newblock A direct formulation for sparse {PCA} using semidefinite programming.
\newblock {\em {SIAM} Rev.}, 49(3):434--448, 2007.

\bibitem[DH14]{DemanetH14}
Laurent Demanet and Paul Hand.
\newblock Scaling law for recovering the sparsest element in a subspace.
\newblock {\em Information and Inference: A Journal of the IMA}, 3(4):295--309,
  2014.

\bibitem[dKNS20]{dOrsiKNS20}
Tommaso d'Orsi, Pravesh~K. Kothari, Gleb Novikov, and David Steurer.
\newblock Sparse {PCA:} algorithms, adversarial perturbations and certificates.
\newblock In {\em 61st {IEEE} Annual Symposium on Foundations of Computer
  Science, {FOCS}}, pages 553--564, 2020.

\bibitem[DKS19]{DiakonikolasKS19}
Ilias Diakonikolas, Weihao Kong, and Alistair Stewart.
\newblock Efficient algorithms and lower bounds for robust linear regression.
\newblock In {\em Proceedings of the Thirtieth Annual {ACM-SIAM} Symposium on
  Discrete Algorithms, {SODA}}, pages 2745--2754, 2019.

\bibitem[DMA97]{davis1997adaptive}
Geoff Davis, Stephane Mallat, and Marco Avellaneda.
\newblock Adaptive greedy approximations.
\newblock {\em Constructive approximation}, 13(1):57--98, 1997.

\bibitem[DR13]{Dymarski2013SparseSM}
Przemyslaw Dymarski and Rafal Romaniuk.
\newblock Sparse signal modeling in a scalable celp coder.
\newblock {\em 21st European Signal Processing Conference (EUSIPCO 2013)},
  pages 1--5, 2013.

\bibitem[EYOR13]{Ens2013OptimizedSW}
Alexander Ens, Adnan Yousaf, Thomas Ostertag, and Leonhard~Michael Reindl.
\newblock Optimized sinus wave generation with compressed sensing for radar
  applications, 2013.

\bibitem[FF93]{frank1993statistical}
LLdiko~E. Frank and Jerome~H. Friedman.
\newblock A statistical view of some chemometrics regression tools.
\newblock {\em Technometrics}, 35(2):109--135, 1993.

\bibitem[FKT15]{FosterKT15}
Dean~P. Foster, Howard~J. Karloff, and Justin Thaler.
\newblock Variable selection is hard.
\newblock In {\em Proceedings of The 28th Conference on Learning Theory,
  {COLT}}, volume~40, pages 696--709, 2015.

\bibitem[FL01]{fan2001variable}
Jianqing Fan and Runze Li.
\newblock Variable selection via nonconcave penalized likelihood and its oracle
  properties.
\newblock {\em Journal of the American statistical Association},
  96(456):1348--1360, 2001.

\bibitem[Fos18]{Fosson2018NonconvexLA}
Sophie~Marie Fosson.
\newblock Non-convex lasso-kind approach to compressed sensing for
  finite-valued signals.
\newblock {\em arXiv: Optimization and Control}, 2018.

\bibitem[GV21]{gupte2021finegrained}
Aparna Gupte and Vinod Vaikuntanathan.
\newblock The fine-grained hardness of sparse linear regression, 2021.

\bibitem[HIM18]{Har-PeledIM18}
Sariel Har{-}Peled, Piotr Indyk, and Sepideh Mahabadi.
\newblock Approximate sparse linear regression.
\newblock In {\em 45th International Colloquium on Automata, Languages, and
  Programming, {ICALP}}, volume 107, pages 77:1--77:14, 2018.

\bibitem[HK70]{hoerl1970ridge}
Arthur~E. Hoerl and Robert~W. Kennard.
\newblock Ridge regression: Biased estimation for nonorthogonal problems.
\newblock {\em Technometrics}, 12(1):55--67, 1970.

\bibitem[JLST21]{jambulapati2021robust}
Arun Jambulapati, Jerry Li, Tselil Schramm, and Kevin Tian.
\newblock Robust regression revisited: Acceleration and improved estimation
  rates.
\newblock {\em arXiv preprint arXiv:2106.11938}, 2021.

\bibitem[Jol95]{jolliffe1995rotation}
Ian~T. Jolliffe.
\newblock Rotation of principal components: choice of normalization
  constraints.
\newblock {\em Journal of Applied Statistics}, 22(1):29--35, 1995.

\bibitem[JTU03]{jolliffe2003modified}
Ian~T. Jolliffe, Nickolay~T. Trendafilov, and Mudassir Uddin.
\newblock A modified principal component technique based on the lasso.
\newblock {\em Journal of computational and Graphical Statistics},
  12(3):531--547, 2003.

\bibitem[KBV09]{koren2009matrix}
Yehuda Koren, Robert Bell, and Chris Volinsky.
\newblock Matrix factorization techniques for recommender systems.
\newblock {\em Computer}, 42(8):30--37, 2009.

\bibitem[KE12]{Karahanoglu2012AOM}
Nazim~Burak Karahanoglu and Hakan Erdogan.
\newblock A* orthogonal matching pursuit: Best-first search for compressed
  sensing signal recovery.
\newblock {\em Digit. Signal Process.}, 22:555--568, 2012.

\bibitem[KKK19]{KarmalkarKK19}
Sushrut Karmalkar, Adam~R. Klivans, and Pravesh Kothari.
\newblock List-decodable linear regression.
\newblock In {\em Advances in Neural Information Processing Systems 32: Annual
  Conference on Neural Information Processing Systems, NeurIPS}, pages
  7423--7432, 2019.

\bibitem[KKLP16]{Keiper2016CompressedSF}
Sandra Keiper, Gitta Kutyniok, Dae~Gwan Lee, and G{\"o}tz~E. Pfander.
\newblock Compressed sensing for finite-valued signals.
\newblock {\em arXiv: Optimization and Control}, 2016.

\bibitem[KKM18]{KlivansKM18}
Adam~R. Klivans, Pravesh~K. Kothari, and Raghu Meka.
\newblock Efficient algorithms for outlier-robust regression.
\newblock In {\em Conference On Learning Theory, {COLT}}, pages 1420--1430,
  2018.

\bibitem[LW15]{loh2015regularized}
Po-Ling Loh and Martin~J. Wainwright.
\newblock Regularized m-estimators with nonconvexity: Statistical and
  algorithmic theory for local optima.
\newblock {\em The Journal of Machine Learning Research}, 16(1):559--616, 2015.

\bibitem[Mag17]{Magdon-Ismail17}
Malik Magdon{-}Ismail.
\newblock Np-hardness and inapproximability of sparse {PCA}.
\newblock {\em Inf. Process. Lett.}, 126:35--38, 2017.

\bibitem[Mah15]{Mahabadi15}
Sepideh Mahabadi.
\newblock Approximate nearest line search in high dimensions.
\newblock In {\em Proceedings of the Twenty-Sixth Annual {ACM-SIAM} Symposium
  on Discrete Algorithms, {SODA}}, pages 337--354, 2015.

\bibitem[MWA06]{MoghaddamWA06}
Baback Moghaddam, Yair Weiss, and Shai Avidan.
\newblock Generalized spectral bounds for sparse {LDA}.
\newblock In {\em Machine Learning, Proceedings of the Twenty-Third
  International Conference {ICML}}, pages 641--648, 2006.

\bibitem[MYL{\etalchar{+}}10]{Meng2010CollaborativeSS}
Jia Meng, Wotao Yin, Husheng Li, Ekram Hossain, and Zhu Han.
\newblock Collaborative spectrum sensing from sparse observations using matrix
  completion for cognitive radio networks.
\newblock {\em 2010 IEEE International Conference on Acoustics, Speech and
  Signal Processing}, pages 3114--3117, 2010.

\bibitem[Nat95]{Natarajan95}
Balas~K. Natarajan.
\newblock Sparse approximate solutions to linear systems.
\newblock {\em {SIAM} J. Comput.}, 24(2):227--234, 1995.

\bibitem[NR12]{Nakarmi2012BCSCS}
Ukash Nakarmi and Nazanin Rahnavard.
\newblock Bcs: Compressive sensing for binary sparse signals.
\newblock {\em MILCOM 2012 - 2012 IEEE Military Communications Conference},
  pages 1--5, 2012.

\bibitem[OWZ15]{ODonnellWZ15}
Ryan O'Donnell, Yi~Wu, and Yuan Zhou.
\newblock Hardness of max-2lin and max-3lin over integers, reals, and large
  cyclic groups.
\newblock {\em {ACM} Trans. Comput. Theory}, 7(2):9:1--9:16, 2015.

\bibitem[RHE14]{Rossi2014SpatialCS}
Marco Rossi, Alexander~M. Haimovich, and Yonina~C. Eldar.
\newblock Spatial compressive sensing for mimo radar.
\newblock {\em IEEE Transactions on Signal Processing}, 62:419--430, 2014.

\bibitem[SBRJ19]{SuggalaBR019}
Arun~Sai Suggala, Kush Bhatia, Pradeep Ravikumar, and Prateek Jain.
\newblock Adaptive hard thresholding for near-optimal consistent robust
  regression.
\newblock In {\em Conference on Learning Theory, {COLT}}, pages 2892--2897,
  2019.

\bibitem[SH08]{shen2008sparse}
Haipeng Shen and Jianhua~Z. Huang.
\newblock Sparse principal component analysis via regularized low rank matrix
  approximation.
\newblock {\em Journal of multivariate analysis}, 99(6):1015--1034, 2008.

\bibitem[SKPB12]{StuderKPB12}
Christoph Studer, Patrick Kuppinger, Graeme Pope, and Helmut B{\"{o}}lcskei.
\newblock Recovery of sparsely corrupted signals.
\newblock {\em {IEEE} Trans. Inf. Theory}, 58(5):3115--3130, 2012.

\bibitem[Tib96]{tibshirani1996regression}
Robert Tibshirani.
\newblock Regression shrinkage and selection via the lasso.
\newblock {\em Journal of the Royal Statistical Society: Series B
  (Methodological)}, 58(1):267--288, 1996.

\bibitem[TWZ{\etalchar{+}}22]{TukanWZBF22}
Murad Tukan, Xuan Wu, Samson Zhou, Vladimir Braverman, and Dan Feldman.
\newblock New coresets for projective clustering and applications.
\newblock In {\em International Conference on Artificial Intelligence and
  Statistics, {AISTATS}}, 2022.

\bibitem[ZH05]{zou2005regularization}
Hui Zou and Trevor Hastie.
\newblock Regularization and variable selection via the elastic net.
\newblock {\em Journal of the royal statistical society: series B (statistical
  methodology)}, 67(2):301--320, 2005.

\bibitem[ZHT06]{zou2006sparse}
Hui Zou, Trevor Hastie, and Robert Tibshirani.
\newblock Sparse principal component analysis.
\newblock {\em Journal of computational and graphical statistics},
  15(2):265--286, 2006.

\bibitem[ZJS20]{zhu2020}
Banghua Zhu, Jiantao Jiao, and Jacob Steinhardt.
\newblock Robust estimation via generalized quasi-gradients.
\newblock {\em CoRR}, abs/2005.14073, 2020.

\end{thebibliography}

\appendix
\section{NP Hardness Result}
We give an alternate proof of NP hardness for robust regression based on exact cover.

\begin{problem}[Exact Cover]
\label{prob:exact:cover}
Given a collection $S$ of subsets of $X$, determine if there exists a sub collection $S'$ of $S$ such that every member of $X$ belongs to exactly one set in $S'$.
\end{problem}

\begin{problem}[Robust Regression, Zero Cost Decision Version]
\label{prob:robust:reg:zero}
Given $A \in \R^{n \times d}, b \in \R^n$, and integer $0 < k \le n$, determine if there exists $T \subset [n]$ satisfying $|T| = k$ such that 
\begin{equation}\label{eq:main}
    \min_{y \in \R^d} \|(Ay-b)_T\| = 0 
\end{equation}
where $(Ax-b)_T$ denotes that we only measure the loss on the coordinates in $T$. The coordinates not in $T$ are called \emph{ignored}.
\end{problem}

\begin{lemma}
Problem~\ref{prob:exact:cover} is reducible to Problem~\ref{prob:robust:reg:zero}.
\end{lemma}
\begin{proof}
Consider an exact cover instance given by $S, X$. Let $n = 2|S| + |X|$ and $d = |S|$. We form the matrix $A \in \R^{n \times d}$ as follows. We have a variable $y_i$ for the $i$ th set in $S$ for all $i$. The first $2|S| \times |S|$ block of $A$ will be the constraints $y_i = 0$ and $1-y_i = 0$. This also defines the $b$ vector for this part of the matrix. The next $|X| \times |S|$ block of $A$ will be the indicator matrix for the sets in $S$. That is, each column will be a $\{0,1\}$ vector indicating which elements of $X$ are in the set corresponding to the column. The part of the $b$ vector for this block of $A$ will all $1's$. Finally, we set $k = |X| + |S|$.

We now claim that Eq.\ \eqref{eq:main} is equal to $0$ iff an exact cover exists. In particular, we claim that Eq.\ \eqref{eq:main} is equal to $0$ iff $y$ is the indicator vector for which sets in $S$ to pick to be part of $S'$, the exact cover. First, note that if an exact cover exists, letting $y_i = 1$ for the sets that are part of $S'$ (and thus ignoring the $y_i = 0$ constraints), and letting $y_i = 0$ for sets that are not part of $S'$ (and again ignoring the $1-y_i = 0$ constraints) results in $(Ay-b)_T = 0$. Note that we have chosen to ignore exactly $n-k = |S|$ many constraints, one for each $y_i$.

We now show the other direction. Suppose that Eq.\ \eqref{eq:main} holds. We first show that $y$ must only have $0,1$ entries. Let's focus on the first $2|S|$ constraints of $A$. If both of the constraints $y_i = 0$ and $1-y_i = 0$ are not ignored for some $i$, then we automatically induce a non zero cost. Since we are assuming Eq.\ \eqref{eq:main} holds, it implies that all variables $y_i$ must only have one of $y_i = 0$ or $1-y_i = 0$ present in $T$ for all $i$ and furthermore, only these types of constraints must be ignored in $T$. Therefore, $y \in \{0,1\}^n$ and $y$ represents an indicator vector for the second $|X| \times |S|$ block of $A$. Since the $b$ vector for this block is the all $1$'s vector, this automatically implies that the sets chosen by $y$ forms an exact cover of $X$, as desired.
\end{proof}

\section{Simple Counter Examples to Natural Algorithms for Robust Regression}\label{sec:counter_examples}

In this section, we present particularly simple counter examples to natural algorithms for robust regression which have also been studied in applied works. 

\paragraph{Greedy algorithm.}
First consider a greedy algorithm which fits a best fit linear regression on all data points and removes the $k$ data points with the largest residuals, such as in Algorithm~\ref{alg:greedy}. 
Variants of this algorithm have been used in applied works such as \cite{broderick2020automatic} to find the `most influential' data points in econometric data analysis.

\begin{figure}[!htb]
\begin{mdframed}
\begin{enumerate}
\item
Given $A\in\mathbb{R}^{n\times d}$ and $b\in\mathbb{R}^d$, $y=\argmin_x\|Ax-b\|_2$.
\item
Let $S$ be the $n\times n$ identity matrix and let $p_1,\ldots,p_k$ be the indices of the $k$ coordinates of $Ay-b$ largest in magnitude. 
\item 
For $i\in[k]$, set $S_{p_i,p_i}=0$. 
\item
Output $S$ and $y$.
\end{enumerate}
\end{mdframed}
\caption{Greedy algorithm for robust regression}\label{alg:greedy}
\end{figure}

Now consider the following $(x,y)$ data pairs: $(0,10), (1,0), (2,0), (3,0)$. We can generalize this example to any total number of data points by having multiple copies of each data point. Consider the simplest $k=1$ case of robust regression where we wish to remove one point to minimize the regression loss. In the example given, it is clear that if we remove $(0,10)$, zero loss is achieved by the line $y = 0$. The best fit on all of the points is given by $y = 7-3x$. We can check the residual for the $(0,10)$ data point is $3$ while the residual for $(1,0)$ is $4$. Thus, the greedy algorithm removes $(1,0)$ which results in a suboptmial algorithm which performs arbitrarily worse compared to the true solution with $0$ loss.

\paragraph{Alternating minimization.}
We now consider another natural algorithm which performs alternating minimization: starting from an arbitrary $S$, it optimizes for $x$ given the choice of $S$. Then using the resulting $x$, it optimizes for $S$ and continues in this loop for a specified number of iterations. 
See Algorithm~\ref{alg:alter:min} for more details. 

\begin{figure}[!htb]
\begin{mdframed}
\begin{enumerate}
\item
Given $A\in\mathbb{R}^{n\times d}$ and $b\in\mathbb{R}^d$, and number of iterations $T$, set $S$ to be an arbitrary $n\times n$ diagonal matrix with $n-k$ ones on the diagonal and $k$ zeros.
\item
While $\#$ of iterations $<T$:
\begin{enumerate}
\item 
Set $y=\argmin_x\|SAx-Sb\|_2$. (Optimize over $x$)
\item Let $p_1, \ldots, p_k$ be the indices of the $k$ coordiantes $Ay-b$ largest in magnitude.
\item For $i \in [k],$ set $S_{p_i, p_i} = 0$. (Optimize over $S$)
\end{enumerate}
\item
Output $S$ and $y=\argmin_x\|SAx-Sb\|_2$.
\end{enumerate}
\end{mdframed}
\caption{Alternating minimization algorithm for robust regression}\label{alg:alter:min}
\end{figure}

This class of algorithms is widely used in practice; for example, it was a key component in the winning submission for the Netflix Prize Competition \cite{koren2009matrix}. Alternating algorithms have also been considered for robust regression in the distributional setting \cite{BhatiaJK15, Bhatia0KK17, SuggalaBR019}. 
Alternating minimization algorithms are especially useful where one is interested in minimizing a complex function of various parameters with the property that minimizing over specific subsets of the variables is tractable. Indeed, this is the case here: given $S$, finding $x$ is just an instance of linear least squares with no restrictions and given $x$, the best $S$ is given by discarding the $k$ datapoints with the largest loss.

Our example for the greedy algorithm again serves as a simple counter example for the proposed alternating minimization algorithm for the most basic case of $k=1$ in the robust regression problem. Suppose we start with the matrix $S$ which removes or ignores the point $(1,0)$. Doing so gives us the best fit line $y = 9.29x-3.57$. However for this line, one can check that the point $(1,0)$ would still have the largest residual among all four points. Therefore, the alternating minimization algorithm would not make any further progress as it would continue to select the point $(1,0)$ to remove in all future iterations. We can check that if we started by removing the point $(2,0)$ instead, the point $(1,0)$ would still have the largest residual among all four data points in the resulting best fit line. Thus, we are back in the first case considered. If we start by removing $(3,0)$, then $(3,0)$ will have the largest residual among all four data points so the alternating minimization algorithm is again stuck. Therefore, the alternating minimization algorithm is guaranteed to return a suboptmial solution if we do not initialize $S$ with the optimal choice.

\section{Polynomial-time Algorithm for Planted Instance of Robust Regression}

In this section, we show that if the columns of the input matrix $A$ are generated from a normal distribution and the measurement vector $b$ has Hamming distance at most $k$ from a planted solution $b'$ that lies in the column span of $A$, then there is a polynomial time algorithm that solves the robust regression problem:

\begin{theorem}
\label{thm:robust:planted}
Let $C$ be a fixed constant and $k\le C\sqrt{n}\log n$. 
Let the columns of an input matrix $A\in\mathbb{R}^{n\times d}$ be drawn independent and identically distributed from $\mathcal{N}(0,I_n)$. 
Let $b'\in\mathbb{R}^n$ lie in the column span of $A$. 
Then given a vector $b$ such that $\|b-b'\|_0\le k$, there exists an algorithm that solves $n$ linear programs and then uses polynomial time to solve the sparse linear regression problem with probability at least $2/3$, i.e., the algorithm finds a diagonal matrix $S\in\mathbb{R}^{n\times n}$ with $n-k$ nonzero entries along that diagonal that are set to $1$ and a vector $x\in\mathbb{R}^d$ such that $\|S(Ax-b)\|=0$. 
\end{theorem}

 Our result is motivated by the following result of \cite{DemanetH14} which solve the problem of sparsest non-zero vector in a subspace in a planted setting as well.

\begin{lemma}[Theorem 1 in \cite{DemanetH14}]
\label{lem:planted:alg}
Given a basis of vectors $w_1,\ldots,w_{d+1}\in\mathbb{R}^n$ for a subspace spanned by vectors $v,v_1,\ldots,v_d\in\mathbb{R}^n$, where $v_i\sim\mathcal{N}(0,I_n)$ for all $i\in[d]$, then there exists an absolute constant $C>0$ and an algorithm that solves $n$ linear programs and uniquely recovers the vector $v$ with probability at least $2/3$, for $\|v\|_0\le C\sqrt{n}\log n$. 
\end{lemma}

\begin{proof}[Proof of Theorem \ref{thm:robust:planted}]
Given a matrix $A$ whose columns $u_1,\ldots,u_d\sim\mathcal{N}(0,I_n)\in\mathbb{R}^n$ and a vector $b\in\mathbb{R}^n$ such that there exists a vector $b'\in\mathbb{R}^n$ in the column span of $A$ with $\|b-b'\|_0\le k$, we construct the vectors $w_1,\ldots,w_{d+1}$ by taking an arbitrary basis over the $d+1$ vectors $b,u_1,\ldots,u_d$. 
The vectors $w_1,\ldots,w_{d+1}$ also form a basis for the subspace spanned by the vectors $b-b',u_1,\ldots,u_d$, since $b'$ is in the column span of $A$ and thus spanned by $u_1,\ldots,u_d$. 
Since $\|b-b'\|_0\le k$, then by Lemma~\ref{lem:planted:alg}, there exists an algorithm that solves $n$ linear programs and uniquely recovers the vector $b-b'$ with probability at least $2/3$. 
Because we are given $b$ as input, we can thus determine the vector $b'$, as well as a vector $x\in\mathbb{R}^d$ such that $Ax=b'$. 
By setting $S$ to be the diagonal matrix $S$ with at most $k$ zeros and at least $n-k$ ones on the diagonal such that the zero entries on the diagonal of $S$ are located precisely in the coordinates for which $b-b'$ is nonzero, then we have $\|S(Ax-b)\|=\|S(b-b')\|=0$, since $S(b-b')=0^n$. 
\end{proof}

\end{document}